\documentclass[conference]{IEEEtran}
\usepackage[english]{babel}
\usepackage{amsthm}
\usepackage{amsmath, amssymb, amsfonts}
\usepackage{bm}
\usepackage{array}
\usepackage{textcomp}
\usepackage{stfloats}
\usepackage{url}
\usepackage{hyperref}
\usepackage{verbatim}
\usepackage{graphicx}
\usepackage{xcolor}
\usepackage{cite}
\usepackage{multirow}
\usepackage{algpseudocode}
\usepackage{balance}
\usepackage[linesnumbered, ruled, vlined]{algorithm2e}

\newtheorem{theorem}{Theorem}

\newtheorem{problem}{Problem}
\newtheorem{definition}{Definition}
\newtheorem{remark}{Remark}

\begin{document}

\title{Large-Scale Multirobot Coverage Path Planning on Grids with Path Deconfliction}

\author{
Jingtao Tang,
Zining Mao, 
Hang Ma
\thanks{The authors are with the School of Computing Science, Simon Fraser University, Burnaby, BC V5A1S6, Canada. {\tt\footnotesize \{jingtao\_tang, zining\_mao, hangma\}@sfu.ca}.}
}

\markboth{Journal of \LaTeX\ Class Files,~Vol.~14, No.~8, August~2021}%
{Jingtao Tang, Zining Mao, Hang Ma, \MakeLowercase{\textit{et al.}}: Large-Scale Multirobot Coverage Path Planning on Grids with Path Deconfliction}


\maketitle

\begin{abstract}
We study Multirobot Coverage Path Planning (MCPP) on a 4-neighbor 2D grid $G$, which aims to compute paths for multiple robots to cover all cells of $G$.
Traditional approaches are limited as they first compute coverage trees on a quadrant coarsened grid $\mathcal{H}$ and then employ the Spanning Tree Coverage (STC) paradigm to generate paths on $G$, making them inapplicable to grids with partially obstructed $2 \times 2$ blocks. To address this limitation, we reformulate the problem directly on $G$, revolutionizing grid-based MCPP solving and establishing new NP-hardness results.
We introduce Extended-STC (ESTC), a novel paradigm that extends STC to ensure complete coverage with bounded suboptimality, even when $\mathcal{H}$ includes partially obstructed blocks. Furthermore, we present LS-MCPP, a new algorithmic framework that integrates ESTC with three novel types of neighborhood operators within a local search strategy to optimize coverage paths directly on $G$. Unlike prior grid-based MCPP work, our approach also incorporates a versatile post-processing procedure that applies Multi-Agent Path Finding (MAPF) techniques to MCPP for the first time, enabling a fusion of these two important fields in Multirobot coordination. This procedure effectively resolves inter-robot conflicts and accommodates turning costs by solving a MAPF variant, making our MCPP solutions more practical for real-world applications.
Extensive experiments demonstrate that our approach significantly improves solution quality and efficiency, managing up to 100 robots on grids as large as $256 \times 256$ within minutes of runtime. Validation with physical robots confirms the feasibility of our solutions under real-world conditions.
A project page with code, demo videos, and additional resources is available at: \url{https://sites.google.com/view/lsmcpp}.
\end{abstract}


\section{Introduction}\label{sec:intro}
Coverage Path Planning (CPP) addresses the problem of determining a path that fully covers a designated workspace~\cite{galceran2013survey}.
This problem is essential for a broad spectrum of robotic applications, from indoor tasks like vacuum cleaning~\cite{wang2024apf} and inspection~\cite{almadhoun2018coverage} to outdoor activities such as automated harvesting~\cite{chen2024optimizing}, planetary exploration~\cite{santra2024risk}, and environmental monitoring~\cite{sudha2024coverage}. 
Multirobot Coverage Path Planning (MCPP) is an extension of CPP tailored for Multirobot systems, aiming to coordinate the paths of multiple robots to collectively cover the given workspace, thereby enhancing both task efficiency~\cite{tang2022learning} and system robustness~\cite{sun2021ft}. A fundamental challenge of MCPP is to balance the costs across multiple robots, often quantified by the \textit{makespan}---the maximum path cost among all robots. This challenge intensifies in large-scale applications where the number of robots and the size of the workspace significantly increase.


\begin{figure}[t]
\centering
\includegraphics[width=\linewidth]{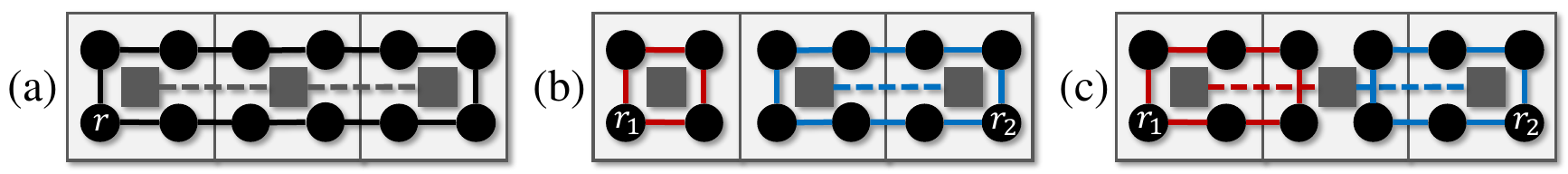}
\caption{Grid-based CPP and MCPP: black circles, dark gray squares, solid lines, and dashed lines represent vertices of $G$, vertices of $\mathcal{H}$, coverage paths on $G$, and spanning edges of $\mathcal{H}$, respectively. Starting at $r$, the robot(s) need to collectively cover each vertex of $G$ and return to $r$, with a cost of $1$ to traverse between adjacent vertices of $G$. (a) Single-robot coverage path generated by STC. (b) Suboptimal $2$-robot coverage paths with a makespan of $8$, resulting from 2 coverage trees that minimize the maximum tree weights. (c) Optimal $2$-robot coverage paths with a makespan of $6$. 
}
\label{fig:STC}
\end{figure}

This article tackles MCPP on a 4-neighbor 2D (edge-)weighted grid graph $G$. The problem is to compute paths for a given set of $k$ robots, each starting from and returning to their given root vertices, to jointly cover all vertices of $G$, a problem analogous to the NP-hard Multiple Traveling Salesman Problem (m-TSP)~\cite{francca1995m} and Vehicle Routing Problem (VRP)~\cite{carlsson2009solving}. 
Existing grid-based MCPP methods~\cite{hazon2005redundancy,zheng2007robot,zheng2010multirobot,tang2021mstc,tang2023mixed} operate on a coarsened grid $\mathcal{H}$, derived from $G$ by grouping every $2\times 2$ adjacent vertices, to compute coverage trees that jointly cover all vertices of $\mathcal{H}$. These methods then apply the Spanning Tree Coverage (STC)~\cite{gabriely2001spanning} paradigm to generate coverage paths on $G$ by circumnavigating the computed coverage trees. However, operating exclusively on the coarsened grid $\mathcal{H}$ has two significant disadvantages. First, it fails in cases where $\mathcal{H}$ is \textit{incomplete}---that is, when any $2\times 2$ blocks contain obstructed grid cells absent from $G$.\footnote{The grid graph $G$ is technically a vertex-induced subgraph of an infinite grid. For clarity and consistency, we refer to the vertices in the infinite grid as "grid cells" throughout this paper. These grid cells can be obstructed, resulting in their absence as vertices in $G$. This distinction helps articulate the difference between the physical presence of cells in the infinite grid and their representation (or lack thereof) in the subgraph $G$. Such terminology is particularly useful in scenarios where contracting a $2\times 2$ block of grid cells does not yield a complete set of four vertices, or a ``hypervertex'', due to the obstruction or absence of one or more grid cells. This approach enables us to discuss the structure and topology of $G$ more precisely, especially in the context of MCPP problems where the presence or absence of grid cells affects path planning and coverage strategies.} Second, even optimal coverage trees on $\mathcal{H}$ do not necessarily result in an optimal MCPP solution (as illustrated in Fig.~(\ref{fig:STC}b) and (\ref{fig:STC}c)), as evidenced by an asymptotic suboptimality ratio of four for makespan minimization~\cite{zheng2010multirobot}, since the paths derived from circumnavigating coverage trees of $\mathcal{H}$ constitute only a subset of all possible sets of coverage paths on $G$. Furthermore, existing grid-based MCPP methods neglect to resolve inter-robot conflicts \cite{stern2019multi} in their computed solutions, leading to potential collisions during execution and thus limiting their practical application in diverse real-world Multirobot scenarios.

\subsection{Contributions}
We revolutionize solving MCPP on grid graphs, overcoming the above limitations through a two-phase approach that first systematically searches for good coverage paths directly on $G$ and subsequently resolves inter-robot conflicts in a post-processing procedure, which is illustrated in Fig.~(\ref{fig:overview}).
Our algorithmic contribution are detailed as follows:
\begin{enumerate}
    \item We propose a novel standalone paradigm called Extended-STC (ESTC), which extends STC to address coverage path-planning problems on any grid graph $G$, applicable both when the quadrant coarsening $\mathcal{H}$ is complete and incomplete. Importantly, we demonstrate that ESTC guarantees complete coverage for both single- and Multirobot settings with bounded suboptimality, rendering it an efficient and versatile solution for coverage path planning.
    \item We design three types of specialized neighborhood operators to facilitate an effective local search process by identifying cost-efficient subgraphs of $G$ that are then used to generate coverage paths for the robots. The strategic integration of these operators significantly improves the efficiency of exploring the solution space. We then demonstrate how to combine these neighborhood operators with iterative calls to the ESTC paradigm to establish our LS-MCPP framework for solving MCPP.
    \item We address inter-robot conflicts in the coverage paths computed by LS-MCPP or other MCPP methods by formulating the problem as a variant of Multi-Agent Path Finding~\cite{stern2019multi}. This variant asks to compute conflict-free continuous-time trajectories for robots as each of them visits an ordered sequence of vertices specified by the MCPP solution. We develop new approaches for solving this variant based on existing MAPF techniques. These approaches integrate the novel Multi-Label Safe Interval Path Planning search that plans a continuous-time trajectory for an individual robot through its ordered sequence of vertices and, as a side benefit, allows our post-processing step to accommodate important kinematic constraints, such as turning costs—--a critical factor in real-world coverage applications~\cite{vandermeulen2019turn,lu2022tmstc,ramesh2022optimal},
\end{enumerate}
We showcase the benefits of our algorithmic pipeline through extensive quantitative experiments on large-scale instances involving up to 100 robots on a $256\times 256$ grid. Key findings include: 
\begin{enumerate}
    \item Our ESTC paradigm results in coverage paths with significantly smaller makespans than existing STC-based paradigms and is effective on any grid graphs $G$, including those with incomplete quadrant coarsening $\mathcal{H}$, for both CPP and MCPP tasks.
    \item Our LS-MCPP framework achieves significantly smaller makespans than state-of-the-art MCPP methods that rely on suboptimal tree coverage computations on $\mathcal{H}$ and requires far less runtime to achieve makespans comparable to or better than those achieved by the optimal tree coverage computation.
    \item Our post-processing procedure can be applied to any MCPP solutions, including not only those generated by LS-MCPP but also those generated by existing MCPP methods, to effectively resolve conflicts between robots and accounts for turning costs, further enhancing the practicability of the solutions.
\end{enumerate}
Additionally, we validate the feasibility of our approach with physical robot deployments, confirming its practical applicability in real-world robotics applications.

\begin{figure}
\centering
\includegraphics[width=0.95\linewidth]{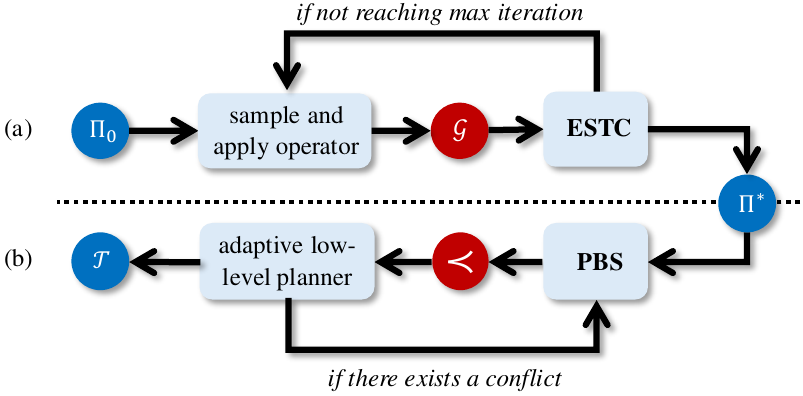}
\caption{The proposed two-phase approach for MCPP in a grid graph with $k$ robots. (a) The LS-MCPP phase: Starting with an initial solution $\Pi_0$, the algorithm iteratively samples operators to update the subgraph set $\mathcal{G}$ and generates $k$ coverage paths using ESTC, ultimately producing a makespan-minimized solution $\Pi^*$. (b) The deconflicting phase: Given $\Pi^*$ (or any other MCPP solution) as input, Priority-Based Search (PBS) explores a tree of priority ordering $\pmb\prec$ while invoking an adaptive low-level planner to generate conflict-free paths, considering only high-priority robot trajectories for each robot. The process terminates with a set $\mathcal{T}$ of conflict-free coverage paths.}
\label{fig:overview}
\end{figure}

This article significantly extends its preliminary conference version \cite{tang2024large} by:
\begin{enumerate}
    \item Providing a reformulation of MCPP directly on an (edge-)weighted grid graph $G$ and a new NP-hardness result.
    \item Improving the original ESTC paradigm by encoding cost differences between circumnavigating paths, with a new theoretical analysis of the properties of ESTC.
    \item Incorporating new vertex-wise operators into the original LS-MCPP framework to facilitate a more effective exploration of the solution space. 
    \item Introducing the new post-processing procedure to resolve conflicts between robots.
    \item Expanding the numerical results and including physical robot experiments to demonstrate the practical applicability of our approach.
\end{enumerate}


\subsection{Organization}
We organize the remainder of this article as follows: We review related problems and methods in Sec.~\ref{sec:related_work}. We present the problem formulations and complexity results in Sec.~\ref{sec:problem}. We introduce the ESTC paradigm in Sec.~\ref{sec:estc}. We detail the specialized neighborhood operators and the LS-MCPP framework in Sec.~\ref{sec:mcpp}). We present our proposed post-processing procedure for resolving inter-robot conflicts in Sec.\ref{sec:deconf}). We present our numerical results in Sec.~\ref{sec:res} and discuss our physical robot experiments in Sec.~\ref{sec:exp}. We conclude our findings and outline future directions in Sec.\ref{sec:conclusion}.

\section{Related Work}\label{sec:related_work}
In this section, we survey related work on MCPP and MAPF.

\subsection{Multirobot Coverage Path Planning (MCPP)}\label{subsec:review_mcpp}
Following the existing taxonomy on CPP~\cite{Tomaszewski-2020-125840}, we categorize MCPP methods into grid-based, cellular decomposition, and global methods based on their approach to workspace decomposition.

\noindent\textbf{Grid-Based Methods:}
Grid-based methods conceptualize the workspace as a grid graph $G$~\cite{kapoutsis2017darp, azpurua2018multi, li2023sp2e}, allowing for the application of various graph algorithms. These methods traditionally operate exclusively on a quadrant coarsened $\mathcal{H}$ derived from $G$ and fail when $\mathcal{H}$ is incomplete. They generalize the Spanning Tree Coverage (STC) paradigm for single-robot CPP in two ways. Single-tree methods\cite{hazon2005redundancy, tang2021mstc} compute a spanning tree on $\mathcal{H}$, generate a coverage path on $G$ that circumnavigate this tree, and then segment this path to distribute among multiple robots. These methods do not offer guarantees on solution quality due to their simplistic approach to path distribution. Multi-tree methods~\cite{zheng2007robot,zheng2010multirobot,tang2023mixed} compute multiple coverage trees, each rooted at the root vertex of a specific robot, to jointly cover all vertices of $\mathcal{H}$, and then generate the coverage paths on $G$ for the robots by circumnavigating their respective trees. In essence, multi-tree methods reduce MCPP to the NP-hard Min-Max Rooted Tree Cover problem~\cite{even2004min, nagamochi2007approximating}, yielding an asymptotic suboptimality ratio of four in makespan minimization on grid graphs. It is important to note that, unlike our direct formulation of MCPP on an edge-weighted $G$, existing grid-based methods typically formulate MCPP on a vertex-weighted $\mathcal{H}$, making a restricted assumption that the weight of each vertex in $\mathcal{H}$ is evenly distributed among its corresponding four vertices in $G$. 

\noindent\textbf{Cellular Decomposition Methods:}
Cellular decomposition methods~\cite{rekleitis2008efficient, xu2014efficient, karapetyan2017efficient} divide the workspace into sub-regions by detecting geometric critical points, using techniques such as trapezoidal decomposition~\cite{latombe1991exact}, boustrophedon decomposition~\cite{choset2000coverage}, and Morse decomposition~\cite{acar2002morse}. 
After dividing the workspace, these methods generate a back-and-forth patterned path to cover each sub-region, often optimizing the path alignment orientation to enhance coverage efficiency~\cite{bochkarev2016minimizing}. 
While these methods are effective in structured environments, their performance diminishes in areas dense with obstacles or in non-rectilinear workspaces due to their reliance on geometric partitioning. Furthermore, they typically assume uniform traversal costs across the workspace, limiting their applicability in scenarios where robots have varying movement costs or task-specific coverage demands.

\noindent\textbf{Global Methods:}
Global methods find significant application in additive manufacturing~\cite{yang2002equidistant, ren2009combined, gibson2021additive} (also known as layered fabrication), where a 3D object is printed layer by layer. Each layer poses a unique coverage problem, necessitating the planning of a space-filling curve to ensure the nozzle covers the entire 2D plane effectively. These methods focus on minimizing path curvature and discontinuities without explicitly decomposing the workspace. Recent advancements include the adaptation of the connected Fermat spiral \cite{zhao2016connected} to reduce energy consumption in CPP tasks over undulating terrains \cite{wu2019energy} and to minimize the makespan in general MCPP scenarios \cite{tang2024multi}. The coverage quality of global methods largely depends on the characteristics of the space-filling curve generated, often necessitating additional trajectory optimization in a post-processing step. 

\subsection{Multi-Agent Path Finding (MAPF)}\label{sec:related_work-MAPF}
MAPF aims to compute conflict-free plans for multiple agents to move from their start vertices to their goal vertices. It is NP-hard to solve optimally for makespan minimization on grid graphs \cite{banfi2017intractability}. We highlight only relevant MAPF variants where each agent is given multiple goal vertices and direct readers interested in broader MAPF discussions to detailed surveys~\cite{stern2019multi, ma2022graph}. OMG-MAPF~\cite{mouratidis2024fools} specifies a total order on the goal vertices given to each agent, while MAPF-PC~\cite{zhang2022multi} specifies precedence constraints---some goal vertex must be visited before another. MG-MAPF~\cite{surynek2021multi, ren2021ms, tang2024mgcbs} does not specify a specific visiting order for the goal vertices given to each agent but instead aims to solve the goal sequencing problem jointly with MAPF. Multi-agent Pickup and Delivery (MAPD)~\cite{ma2017lifelong,liu2019task, xu2022multi} does not pre-assign goal vertices to agents and aims to solve the goal allocation and sequencing problem jointly with MAPF. All the above MAPF variants assume unit action costs and discrete time steps and do not directly apply to the problem of deconflicting MCPP solutions on graphs with nonuniform edge weights.

\section{MCPP and Its NP-Hardness}\label{sec:problem}
We consider a set $I=\{1,2,\ldots ,k\}$ of $k$ robots operating on a four-neighbor connected undirected 2D grid graph $G=(V,E)$, where $V$ is the set of vertices and $E$ is the set of edges connecting each vertex to its top, bottom, left, and right neighbors if they exist.
We assign each edge $e\in E$ with a non-negative weight $w_e$ and simply set $w_e=1$ for every $e\in E$ when considering an unweighted graph $G$.
We specify a set $R=\{r_i\}_{i=1}^k\subseteq V$ of root vertices, where each robot $i$ starts at vertex $r_i$. We use the tuple $(G, I, R)$ to denote an MCPP instance.

A path $\pi_i$ for robot $i$ is defined as a finite ordered sequence of vertices $(v_1,v_2,\ldots ,v_{|\pi_i|})$ such that $v_1=v_{|\pi_i|}=r_i$~\footnote{The requirement that each robot $i$ must start and end at vertex $r_i$ is also known as the \textit{cover-and-return} setting later in~\cite{zheng2010multirobot}, which naturally aligns with STC (and also our ESTC) algorithm in Sec.~\ref{sec:estc} since its coverage path forms a ``loop''. A post-processing procedure can be found in~\cite{zheng2007robot} to convert a STC-like path for the \textit{cover-without-return} setting.} and $(v_{j-1},v_{j})\in E$ for each $j=2,3,\ldots ,|\pi_i|$.
The cost of any path $\pi$ is defined as $c(\pi)=\sum_{e\in \pi}w_e$. For clarity, we let $V(\pi)$ and $\mathcal{E}(\pi)$ denote the set of unique vertices in $\pi$ and the multiset of edges in $\pi$ that possibly contains repeated edges, respectively.
The MCPP problem is defined as follows:
\begin{problem}[MCPP]\label{problem:_mcpp}
Find a set $\Pi=\{\pi_i\}_{i=1}^k$ of paths such that $\bigcup_{i=1}^{k} V(\pi_i) = V$.
\end{problem}
The above problem formulation aligns with the \textit{cover with return} coverage setting~\cite{zheng2010multirobot}, where each robot must start from and return to its given root vertex.
The quality of an MCPP solution is evaluated using the \textit{makespan} metric, denoted by $\theta(\Pi)=\max\{c(\pi_1), c(\pi_2),\ldots ,c(\pi_k)\}$.


\begin{remark}
Following conventional formulations, our MCPP problem definition intentionally excludes inter-robot deconfliction (see Definition~\ref{def:conflict}). However, practical deployment demands collision-free paths. Previous coupled approaches for multi-goal sequencing MAPF~\cite{ren2021ms,ren2023cbss}, a variant of our problem, typically scale only to a few dozen goals (vertices to be covered) and often have low success rates. Our two-phase approach addresses these limitations by first computing an MCPP solution without enforcing collision-free paths (Sec.\ref{sec:mcpp}) and then deconflicting the solution via MAPF techniques (Sec.\ref{sec:deconf}). This decoupled approach not only scales efficiently but is also more generally applicable: Any MCPP method can be used in the first phase, and our post-processing procedure can subsequently handle inter-robot conflicts.
\end{remark}

MCPP can be viewed as an extension of min-max m-TSP on grid graphs where multiple depots exist and each city is allowed to be visited more than once.
Given that the Hamiltonian cycle problem on grid graphs is NP-complete~\cite{itai1982hamilton}, both (single-robot) CPP and thus MCPP are NP-hard, which can be derived through a straightforward reduction from the Hamiltonian cycle problem.
\begin{theorem}~\label{theo:hardness-p1}
It is NP-hard to find a coverage path with the smallest cost on an unweighted grid graph.
\end{theorem}
\begin{proof}
The Hamiltonian cycle problem, known to be NP-complete on a grid graph $G=(V,E)$~\cite{itai1982hamilton}, can be directly reduced to the decision version of CPP that asks to determine whether there is a coverage path on the same grid graph $G$ with a cost not exceeding $|V|$.
\end{proof}
We now examine MCPP for a more restricted class of grid graphs. Although the Hamiltonian cycle problem can be solved in polynomial time~\cite{umans1997hamiltonian} on solid grid graphs, which are grid graphs without holes, the scenario changes when considering MCPP. Specifically, inspired by the NP-hardness proof for a specialized MCPP variant defined on a vertex-weighted grid graph~\cite{zheng2010multirobot}, we extend this complexity result to our more broadly defined MCPP on solid grid graphs through the following theorem.
\begin{theorem}~\label{theo:hardness-p2}
It is NP-hard to find an MCPP solution with the smallest makespan, even on a solid unweighted grid graph.
\end{theorem}
\begin{proof}
We prove by showing the NP-completeness of a decision version of MCPP, that is, determining whether a given MCPP instance has a solution with a makespan not exceeding a given positive integer.
One can evaluate whether the makespan of a given solution exceeds the given positive integer in polynomial time, implying that the MCPP problem is in NP.
We now reduce the MCPP problem to the NP-complete \textit{3-Partition} problem~\cite{gary1979computers} in polynomial time.
Specifically, given a positive integer $B$ and a set $S=\{a_i\}_{i=1}^{3k}$ of positive integers, where each $a_i\in (\frac{B}{4}, \frac{B}{2})$ and $\sum_{i=1}^{3k}a_i=kB$, the problem of 3-Partition is to determine if there exists an evenly partition $S$ into $k$ subsets $S_1,S_2,\ldots ,S_k$, each summing to $B$.

Given a 3-Partition instance, we construct an MCPP instance on an unweighted grid graph as shown in Fig.~(\ref{fig:proof}). This graph has a central horizontal ``corridor'' with $4k$ vertices $r_k,r_{k-1},\ldots ,r_1,v_1,v_2,\ldots ,v_{3k}$ and $3k$ vertical ``tunnel'', each connected to a vertex $v_i$. The $i$-th tunnel contains $4ka_i$ vertices and extends $v_i$ upwards if $i$ is odd and downwards if even.
Robots $1,2,\ldots ,k$ start at their root vertices $r_1,r_2,\ldots ,r_k$, respectively.

We now show that the makespan of the above MCPP instance is at most $8kB+8k-2$ if and only if $S$ can be evenly partitioned into $k$ subsets:

\noindent$\bullet$ ``\textit{If}'' direction:
Given an even partition of $S$ into $k$ subsets $S_j$ where $\sum_{i\in S_j} a_i=B$, we assign each robot $j$ to cover the $i$-th tunnel for each $i\in S_j$. Each robot needs at most $\sum_{i\in S_j}4ka_i\times 2=8kB$ steps to cover its tunnels and at most $(4k-1)\times 2$ steps to traverse the corridor (between $r_k$ to $v_{3k}$). 
Therefore, the cost of any robot is at most $8kB+8k-2$, ensuring the makespan of the resulting MCPP solution does not exceed $8kB+8k-2$.

\noindent$\bullet$ ``\textit{Only if}'' direction:
Given an MCPP solution where the cost of any robot is at most $8kB+8k-2$, we define $S_j$ as the set of tunnel indices where the innermost vertex of each tunnel is first covered by robot $j$, resulting in a partition of $k$ subsets $S_1,S_2,\ldots ,S_k$.
Since each $a_i\in (\frac{B}{4}, \frac{B}{2})$ and $|S|=3k$, each robot must cover exactly $3$ tunnels.
Therefore, the cost for any robot is at least $\sum_{i\in S_j}(4ka_i\times 2)+3\times 2=8k\sum_{i\in S_j}a_i + 6$, calculated as the cost of robot 1 covering the first three tunnels. On the other hand, the cost of any robot is at most $8kB+8k-2$ in the given MCPP solution.
Since $a_i$, $B$ and $k$ are all positive integers, chaining the above two inequalities yields $\sum_{i\in S_j}a_i\leq B+1-\frac{1}{k}\leq B$ for any $j$.
Since $\sum_{j=1}^k\sum_{i\in S_j}a_i=\sum_{i=1}^{3k}a_i=kB$, $\sum_{i\in S_j}a_i=B$ must hold to satisfy $\sum_{i\in S_j}a_i\leq B$, for any $j$. Therefore, $S_1, S_2,\ldots ,S_k$ is an even partition of $S$.
\end{proof}

\begin{figure}[t]
    \centering
    \includegraphics[width=0.6\linewidth]{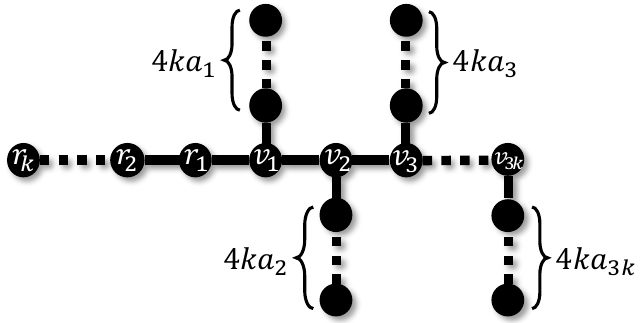}
    \caption{An MCPP instance reducible from the \textit{3-Partition} problem.}
    \label{fig:proof}
\end{figure}


\section{Extended Spanning Tree Coverage (ESTC)}\label{sec:estc}
In this section, we present our new ESTC paradigm  that advances standard STC-based paradigms by addressing coverage problems on any given connected grid graph, even when the quadrant coarsening $\mathcal{H}$ is incomplete. We first describe how ESTC solves (single-robot) CPP, a special case of MCPP (Problem~\ref{problem:_mcpp}) when $k=|I|=1$.
Solving CPP efficiently is fundamental to our local search framework for MCPP detailed in Sec.~\ref{sec:mcpp}. We then analyze the properties and time complexity of ESTC and introduce two local optimizations to improve ESTC solutions. We finally discuss how ESTC can be integrated into existing MCPP methods to effectively handle incomplete $\mathcal{H}$.

ESTC extends both the STC paradigm for offline CPP and the Full-STC paradigm~\cite{gabriely2002spiral} for online CPP, which operate on a quadrant coarsened grid $\mathcal{H}$ derived from $G$. These paradigms typically evaluate the traversal costs by decomposing the given hypervertex weights of $\mathcal{H}$ into vertex weights for $G$ to maintain solution optimality.
However, this approach does not adequately capture the actual traversal costs incurred by a robot, which are better represented by the edges traversed, as defined in our formulation.

\subsection{Algorithm Description}\label{subsec:estc}
ESTC operates on a \textit{hypergraph} graph $H=(V_h,E_h)$, derived from the grid graph $G=(V,E)$. 
As shown in Fig.~(\ref{fig:hypergraph}), similar to standard STC-based paradigms, ESTC attempts to contract neighboring vertices of $V$ corresponding to a $2\times 2$ block of grid cells into a \textit{hypervertex} of $V_h$. 
However, unlike the quadrant coarsened grid $\mathcal{H}$ used by standard STC-based paradigms, the hypergraph $H$ effectively addresses the disconnections caused by partially obstructed $2\times 2$ blocks. 
A \textit{hyperedge} connects two hypervertices only if there is at least one adjacent vertex pair from these hypervertices. A special case in hypervertex construction occurs when a $2\times 2$ block contains only two diagonal vertices. 
In this case, each diagonal vertex forms its own hypervertex, resulting in two nonadjacent hypervertices (see Fig.~(\ref{fig:hypergraph})). 
A hypervertex $\delta$ is \textit{incomplete} if it is formed by fewer than four vertices; otherwise, it is \textit{complete}.
The graph $G$ is contracted left to right and bottom to top by default during the hypergraph construction.
For clarity, we let $\delta_v \in V_h$ denote the hypervertex contracted from $v \in V$ and $\varepsilon = (\delta_u, \delta_v) \in E_h$ denote the hyperedge connecting $\delta_u,\delta_v$.

\begin{figure*}
\centering
\includegraphics[width=0.9\linewidth]{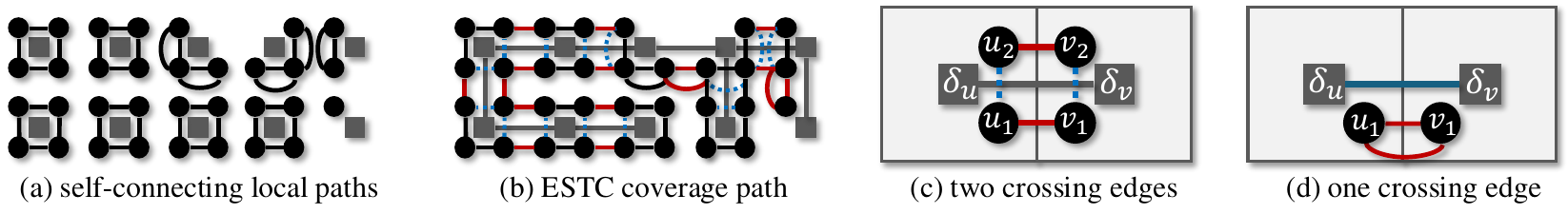}
\caption{Illustration of the ESTC algorithm. (a) The initial local paths (black lines). 
(b) The resulting ESTC coverage path by applying the rerouting rule on each hyperedge of a hypergraph spanning tree.
(c)(d) For each hyperedge (gray solid line) $\varepsilon=(\delta_u,\delta_v)\in E_h$, the rerouting rule adds $E^+_\varepsilon$ (red solid lines) to and removes $E^-_\varepsilon$ (blue dashed lines) from the local paths of $\delta_u$ and $\delta_v$, depending on the number of the crossing edges $(u_1,v_1)$ and $(u_2,v_2)$ of $\varepsilon$.
}
\label{fig:optimal_circum}
\end{figure*}

\begin{figure}
\centering
\includegraphics[width=0.8\linewidth]{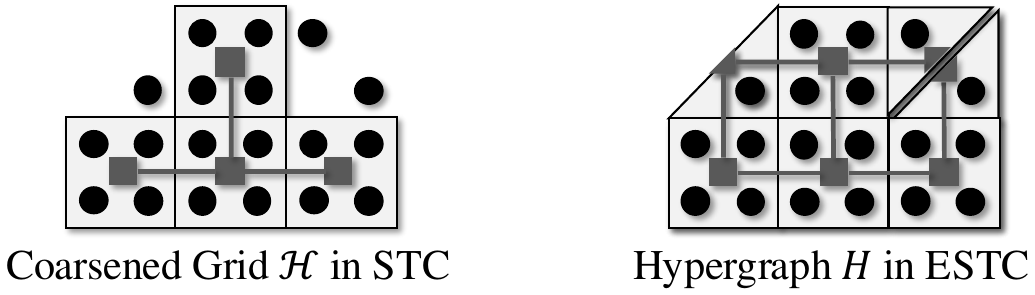}
\caption{The coarsened grid $\mathcal{H}$ in STC and the hypergraph $H$ in ESTC (both in gray markers and lines) for an input grid graph (black circles). The former yields incomplete coverage while the latter ensures complete coverage.}
\label{fig:hypergraph}
\end{figure}

We describe the pseudocode for ESTC in Alg.~\ref{alg:estc}.
ESTC begins with the hypergraph construction [Line~\ref{alg:estc:build_H}] and finds the Minimum Spanning Tree (MST) of $H$.
Specifically, ESTC employs Kruskal's algorithm~\cite{kruskal1956shortest} to find the MST [Line~\ref{alg:estc:find_mst}] since the hyperedge weights (will be defined later) can be negative.
It then initializes a self-connecting local path $\pi_\delta$ for each $\delta_v\in V_h$ [Line~\ref{alg:estc:init}], as shown in Fig.~(\ref{fig:optimal_circum}a).
Note that $\pi_\delta$ is the optimal coverage path for the vertices with respect to hypervertex $\delta$.
Then, for each hyperedge $\varepsilon=(\delta_u,\delta_v)$ in the MST $M$, a rerouting rule is applied to connect $\pi_{\delta_u}$ and $\pi_{\delta_v}$ to construct a coverage path for the vertices in $\delta_u$ and $\delta_v$ [Lines~\ref{alg:estc:connect_loop}-\ref{alg:estc:connect}].
The resulting ESTC coverage path $\pi$ is returned after all the local paths are connected one to another [Line~\ref{alg:estc:return}], as illustrated in Fig.~(\ref{fig:optimal_circum}b).


\noindent\textbf{Rerouting Rule:}
The rerouting rule aims to connect pairwise self-connecting local paths of a hyperedge.
We categorize each hyperedge $\varepsilon=(\delta_u,\delta_v)\in E_h$ based on the number of \textit{crossing edges} between the corresponding vertices within the blocks of $\delta_u, \delta_v$.
A crossing edge is an edge $(u,v)\in E$ where $u$ and $v$ are contracted into adjacent hypervertices $\delta_u$ and $\delta_v$, respectively.
Denote $x_\varepsilon$ as the number of crossing edges for a hyperedge $\varepsilon$, we define two multisets for any hyperedge $\varepsilon$, $E^+_\varepsilon=\{(u_1,v_1),(u_2,v_2)\}$ and $E^-_\varepsilon=\{(u_1,u_2),(v_1,v_2)\}$ if $x_\varepsilon=2$; otherwise when $x_\varepsilon=1$, $E^+_\varepsilon=\{(u_1,v_1),(v_1,u_1)\}$ and $E^-_\varepsilon=\emptyset$.
Note that $x_\varepsilon=0$ implies $(\delta_u,\delta_v)\notin E_h$.
As illustrated in Fig.~(\ref{fig:optimal_circum}c) and (\ref{fig:optimal_circum}d), $E^+_\varepsilon$ and $E^-_\varepsilon$ are represented in red solid lines and blue dashed lines, respectively.
The rerouting rule for every $\varepsilon$ is then defined to add $E^+_\varepsilon$ to connect the two local paths while removing $E^-_\varepsilon$.
It ensures that each vertex has an even degree with the least vertex duplication, making the connected path a valid path.
The above rerouting rule can be interpreted as a combination of the standard STC circumnavigating pattern when $x_\varepsilon=2$ and the Full-STC rerouting rule when $x_\varepsilon=1$. 
It is always feasible as the construction of $H$ guarantees no vertex within any incomplete hypervertex is isolated, thereby ensuring the path remains feasible and continuous across any grid graphs.

\begin{algorithm}[t]
\DontPrintSemicolon
\linespread{0.95}\selectfont
\caption{Extended-STC (ESTC)}\label{alg:estc}
\SetKwInput{KwInput}{Input}
\KwInput{grid graph $G=(V,E)$}
construct hypergraph $H=(V_h,E_h)$ from $G$\;\label{alg:estc:build_H}
$M\gets$ find MST of $H$ via Kruskal's algorithm\;\label{alg:estc:find_mst}
$\pi_\delta\gets$ initial self-connecting local path for each $\delta\in V_h$\;\label{alg:estc:init}
$\pi\gets$ an empty path\;
\For{$\varepsilon=(\delta_u, \delta_v)\in$ the set of hyperedges in $M$}
{\label{alg:estc:connect_loop}
connect $\pi_{\delta_u},\pi_{\delta_v}$ via rerouting rule and add it to $\pi$\;\label{alg:estc:connect}
}
\Return ESTC coverage path $\pi$\;\label{alg:estc:return}
\end{algorithm}

\noindent\textbf{Hyperedge Weights:}
The ESTC algorithm has a good property: Given any two spanning trees $T$ and $T'$ with their resulting ESTC paths $\pi$ and $\pi'$, respectively, their path difference (represented using the multisets of their edges) is a constant, given as follows:
\begin{align}\label{eqn:path_diff}
\begin{split}
\mathcal{E}(\pi)\setminus \mathcal{E}(\pi')=\bigcup_{\varepsilon\in\Delta E_h} E^+_\varepsilon \cup\bigcup_{\varepsilon'\in\Delta E'_h}E^-_{\varepsilon'}&\quad\text{and}\\
\mathcal{E}(\pi')\setminus \mathcal{E}(\pi)=\bigcup_{\varepsilon'\in\Delta E'_h} E^+_{\varepsilon'} \cup\bigcup_{\varepsilon\in\Delta E_h}E^-_{\varepsilon}&,
\end{split}
\end{align}
where $\Delta E_h=E_h(T)\setminus E_h(T')$ and $\Delta E'_h=E_h(T')\setminus E_h(T)$ represent the two hyperedge difference sets between $T$ and $T'$.
We then define the weight $w_\varepsilon$ for hyperedge $\varepsilon$ as below:
\begin{align}\label{eqn:weight_def}
w_\varepsilon=\sum_{e\in E^+_\varepsilon}w_e-\sum_{e\in E^-_\varepsilon}w_e.
\end{align}

The definition of $w_\varepsilon$ implicitly encodes the ESTC path cost difference of whether $\varepsilon$ exists in the spanning tree.
In Theorem~\ref{theo:weight_def}, we validate that with the defined hyperedge weights in Eqn.~(\ref{eqn:weight_def}), an MST of hypergraph $H$ always results in an optimal ESTC circumnavigating coverage path on $G$.

\subsection{Theoretical Analysis}
ESTC is guaranteed to generate a path that covers all vertices of a connected grid graph $G$ based on the argument used in the proof of Lemma 3.1 in~\cite{gabriely2002spiral} since the argument is not affected by the addition of edge weights.
\begin{theorem}\label{theo:complete}
ESTC achieves complete coverage on a connected grid graph $G$.
\end{theorem}

We now analyze the time complexity of ESTC below.
\begin{theorem}\label{theo:time_comp}
The time complexity of ESTC (Alg.~\ref{alg:estc}) on a grid graph $G=(V,E)$ is $O(|V|\log|V|)$.
\end{theorem}
\begin{proof}
To construct the hypergraph $H=(V_h,E_h)$ from $G=(V,E)$ as described in Sec.~\ref{subsec:estc}, ESTC first iterates through $V$ to build $V_h$ and then through $V_h$ to build $E_h$.
Given that $|V_h|$ is proportional to $|V|$, the time complexity of this construction is $O(|V|)$.
To find the MST of $H$, ESTC uses Kruskal's algorithm, which incurs a time complexity of $O(|E|\log|E|)$.
Initializing the local paths and connecting them via the rerouting rule entail a time complexity of $O(|V|+|E|)$.
As $O(|E|)=O(|V|)$ for grid graphs, the overall time complexity of ESTC combines to $O(|V|\log|V|)$.
\end{proof}

The following theorem validates that our definition of hyperedge weights ensures that the ESTC circumnavigating path derived from an MST of $H$ has the smallest cost among those derived from any spanning trees of $H$.

\begin{theorem}\label{theo:weight_def}
The cost of the circumnavigating path derived from a minimum spanning tree (MST) of the hypergraph $H$ is no greater than that from any other spanning tree of $H$.
\end{theorem}
\begin{proof}
For hypergraph $H=(V_h,E_h)$, consider its two spanning trees $T$ and $T'$ and their corresponding ESTC circumnavigating paths $\pi$ and $\pi'$, respectively.
With Eqn.~(\ref{eqn:path_diff}), the path cost difference $\Delta c$ between $\pi$ and $\pi'$ can be computed as:
\begin{align*}
\Delta c&=c(\pi)-c(\pi')
=\sum_{e\in \mathcal{E}(\pi)\setminus \mathcal{E}(\pi')} w_e-\sum_{e\in \mathcal{E}(\pi')\setminus \mathcal{E}(\pi)} w_e\\
&=\sum_{\varepsilon\in\Delta E_h}\sum_{e\in E^+_\varepsilon} w_e + \sum_{\varepsilon'\in\Delta E'_h}\sum_{e\in E^-_{\varepsilon'}} w_e \\ 
&-\sum_{\varepsilon'\in\Delta E'_h}\sum_{e\in E^+_{\varepsilon'}} w_e - \sum_{\varepsilon\in\Delta E_h}\sum_{e\in E^-_{\varepsilon}} w_e.
\end{align*}
With Eqn.~(\ref{eqn:weight_def}) substituted into the above, we have:
\begin{align}\label{eqn:cost_diff}
\begin{split}
\Delta c
&=\sum_{\varepsilon\in\Delta E_h}\left[\sum_{e\in E^+_\varepsilon} w_e -\sum_{e\in E^-_{\varepsilon}} w_e\right]\\
&-\sum_{\varepsilon'\in\Delta E'_h}\left[\sum_{e\in E^+_{\varepsilon'}} w_e -\sum_{e\in E^-_{\varepsilon'}} w_e\right]\\
&=\sum_{\varepsilon\in\Delta E_h}w_\varepsilon -\sum_{\varepsilon'\in\Delta E'_h} w_{\varepsilon'}.
\end{split}
\end{align}
If $T$ is an MST and $T'$ is an arbitrary spanning tree, we know that $\forall\varepsilon\in\Delta E_h \subseteq E_h(T)$ and $\varepsilon'\in\Delta E'_h \subseteq E_h(T')$, we have $w_\varepsilon \leq w_{\varepsilon'}$, where $E_h(T)$ and $E_h(T')$ are the edge sets of $T$ and $T'$, respectively.
Substituting it into Eqn.~(\ref{eqn:cost_diff}), we have $\Delta c\leq 0$ for an MST $T$ and an arbitrary spanning tree $T'$.
\end{proof}

The following theorem shows that the optimal circumnavigating path on $H$ yields an optimal CPP solution when all hypervertices are complete and a bounded suboptimal CPP solution in general.

\begin{theorem}\label{theo:bounded_subopt}
ESTC returns a coverage path with a cost at most $2\cdot\frac{w_\text{max}}{w_\text{min}}\cdot(1+\frac{n_c-1}{|V|})$ times the cost of the optimal single-robot coverage path, where $n_c$ is the number of complete hypervertices and $w_\text{max}$ and $w_\text{min}$ are the maximum and minimum hyperedge weights, respectively. 
\end{theorem}
\begin{proof}
Consider a CPP instance $(G, I, R)$, where $G$ is contracted into hypergraph $H=(V_h,E_h)$ with $n_c$ complete hypervertices.
For any optimal valid coverage path $\pi^*$, it follows that $V(\pi)=V$ and both the first and last vertices of $\pi$ are the root vertex, implying $|\mathcal{E}(\pi^*)|\geq|V|$, with equality only if there are no incomplete hypervertices in $H$. The suboptimality ratio $\rho$ for any ESTC path $\pi$ resulted from MST $T$ is given by:
\begin{align*}
\rho=\frac{c(\pi)}{c(\pi^*)}=\frac{\sum_{e\in \mathcal{E}(\pi)}w_e}{\sum_{e\in \mathcal{E}(\pi^*)}w_e}\leq\frac{|\mathcal{E}(\pi)|\cdot w_\text{max}}{|\mathcal{E}(\pi^*)|\cdot w_\text{min}}\leq\frac{|\mathcal{E}(\pi)|\cdot w_\text{max}}{|V|\cdot w_\text{min}}.
\end{align*}
We now categorize $V_h$ into four disjoint sets such that $V_h=\bigcup_{i=1}^4 V_h^{(i)}$ and each $V_h^{(i)}$ contains hypervertices with $i$ vertices located in it.
As illustrated in Fig.~(\ref{fig:optimal_circum}a), for any $\delta\in V_h^{(i)}$, the initial ESTC self-connecting local path $|\pi_\delta|=4$ if $i=4$ or $3$, $|\pi_\delta|=2$ if $i=2$, and $|\pi_\delta|=0$ if $i=1$.
We categorize the hyperedge set $E_h(T)$ of MST $T$ into $E_h^{(2)}(T)$ and $E_h^{(1)}(T)$ according to the number of crossing edges as described in Fig.~(\ref{fig:optimal_circum}c) and (\ref{fig:optimal_circum}d).
By the initialization of the self-connecting local paths and the rerouting rule, the cardinality of the multiset $\mathcal{E}(\pi)$ can be computed as:
\begin{align}
|\mathcal{E}(\pi)|=4|V_h^{(4)}|+4|V_h^{(3)}| + 2|V_h^{(2)}| + 2|E_h^{(1)}(T)|.
\end{align}
As $|V|=\sum_{i=1}^4 i\cdot|V_h^{(i)}|$, the bound can be further refined as:
\begin{align}\label{eqn:estc_ub_1}
\begin{split}
\rho&\leq\frac{w_\text{max}}{w_\text{min}}\cdot\left[1+\frac{|V_h^{(3)}|-|V_h^{(1)}|+2|E_h^{(1)}(T)|}{|V|}\right]\\
\end{split}.
\end{align}
Substituting $|E_h^{(1)}(T)|\leq|E_h(T)|=|V_h|-1=\sum_{i=1}^4 |V_h^{(i)}|-1$ and $n_c=|V_h^{(4)}|$ into Eqn.~(\ref{eqn:estc_ub_1}), we obtain a simpler yet looser bound:
\begin{align}\label{eqn:estc_ub_2}
\rho\leq2\cdot\frac{w_\text{max}}{w_\text{min}}\cdot\left(1+\frac{n_c-1}{|V|}\right).
\end{align}
\end{proof}
\begin{remark}
With the tighter bound in Eqn.~(\ref{eqn:estc_ub_1}), we can see that ESTC is optimal ($\rho=1$) when $G$ is unweighted ($\frac{w_\text{max}}{w_\text{min}}=1$) and all hypervertices are complete ($|V_h^{(3)}|=|V_h^{(1)}|=|E_h^{(1)}(T)|=0$).
The bound in Eqn.~(\ref{eqn:estc_ub_2}) may appear relatively loose.
However, the last multiplier $(1+\frac{n_c-1}{|V|})\in[1,2)$ as $n_c\leq |V|$.
In fact, it can be seen that as $n_c$ approaches $0$, the bound in Eqn.~(\ref{eqn:estc_ub_2}) gets tighter and closer to $2\cdot\frac{w_\text{max}}{w_\text{min}}$.
This is because when $H$ has more incomplete hypervertices, the optimal solution $\pi^*$ necessitates a rerouting rule, such as that in Full-STC, to revisit edges through incomplete hypervertices to achieve complete coverage.
\end{remark}

\subsection{Local Optimizations}\label{subsec:local_opt}
We introduce two efficient local optimizations for ESTC, namely \textit{Parallel Rewiring} and \textit{Turn Reduction}.
Parallel Rewiring operates as a post-processing procedure to refine the coverage path produced by ESTC, by strategically adjusting path segments locally to reduce the path cost.
Turn Reduction serves as a secondary objective during the execution of ESTC (see Alg.~\ref{alg:estc}) by reducing the number of turns in the path, which is particularly beneficial in environments where turning is costly or undesirable. While these local optimizations do not theoretically improve the upper bound on the suboptimality of ESTC in Theorem~\ref{theo:bounded_subopt}, empirical observations (see Sec.~\ref{subsec:estc}) often demonstrate a notable improvement in solution quality.

\noindent\textbf{Parallel Rewiring:}\label{subsubsec:parallel_rewiring}
As the original STC only explores the solution space of circumnavigating coverage paths, we develop two parallel rewiring procedures to refine any ESTC solution $\pi$ using non-circumnavigating local path segments. 
Type-A Parallel Rewiring searches for every subsequence in $\pi$ matching the pattern $(\ldots ,v,u,\ldots ,v,t,s,\ldots ,t,\ldots )$ and collapses it into $(\ldots ,v,u,\ldots ,s,\ldots t,\ldots )$ if the resulting cost difference $w_{(u,s)}-[w_{(u,v)}+w_{(v,t)}+w_{(t,s)}]<0$.
As shown in Fig.~(\ref{fig:local_opt}a), each collapsing operation effectively removes the intermediate vertices $v,t\in V$ without losing any vertex coverage.
Type-B parallel rewiring identifies every pair of two vertically (or horizontally) parallel edges $(u_1,u_2), (v_1,v_2)\in \mathcal{E}(\pi)$ where $(u_1,u_3), (u_2,u_4), (s_1, u_3), (s_2, u_4)\in \mathcal{E}(\pi)$ are aligned horizontally (or vertically).
As shown in Fig.~(\ref{fig:local_opt}b), it then replaces the edges $(u_1,u_3),(u_2,u_4),(v_1,v_2)$ with $(u_3,u_4),(u_1,v_1),(u_2,v_2)$ if the resulting cost difference $[w_{(u_3,u_4)}+w_{(u_1,v_1)}+w_{(u_2,v_2)}]-[w_{(u_1,u_3)}+w_{(u_2,u_4)}+w_{(v_1,v_2)}]<0$.
Both types have a time complexity of $O(|V|^2)$ with a naive implementation involving two-layer nested iterations through $V$.



\noindent\textbf{Turn Reduction:}
ESTC can be seamlessly integrated with a secondary objective of turn reduction in addition to its primary objective of path cost minimization.
Specifically, when ESTC constructing the MST, it prioritizes the alignment of hyperedges. As the hypergraph $H=(V_h, E_h)$ is a four-neighbor 2d grid, its hyperedges can be categorized by their orientation as vertical or horizontal.
To align with a given target orientation, the ESTC prioritizes the hyperedges traversal in Kruskal's algorithm (line~\ref{alg:estc:find_mst} of Alg.~\ref{alg:estc}) to favor a hyperedge where (1) its orientation matches the given target orientation and (2) sum of the degrees of its two endpoints is the smallest.
The first prioritization ensures orientation consistency as much as possible, whereas the second prioritization aims to extend the MST into less connected regions of $H$ for better alignment.
Turn Reduction requires adding only two priority keys when constructing the MST and thus maintains the same time complexity for ESTC as demonstrated in Theorem~\ref{theo:time_comp}.
Fig.~(\ref{fig:local_opt}c) shows an example of applying Turn Reduction to an ESTC coverage path.

\begin{figure}[t]
\centering
\includegraphics[width=\columnwidth]{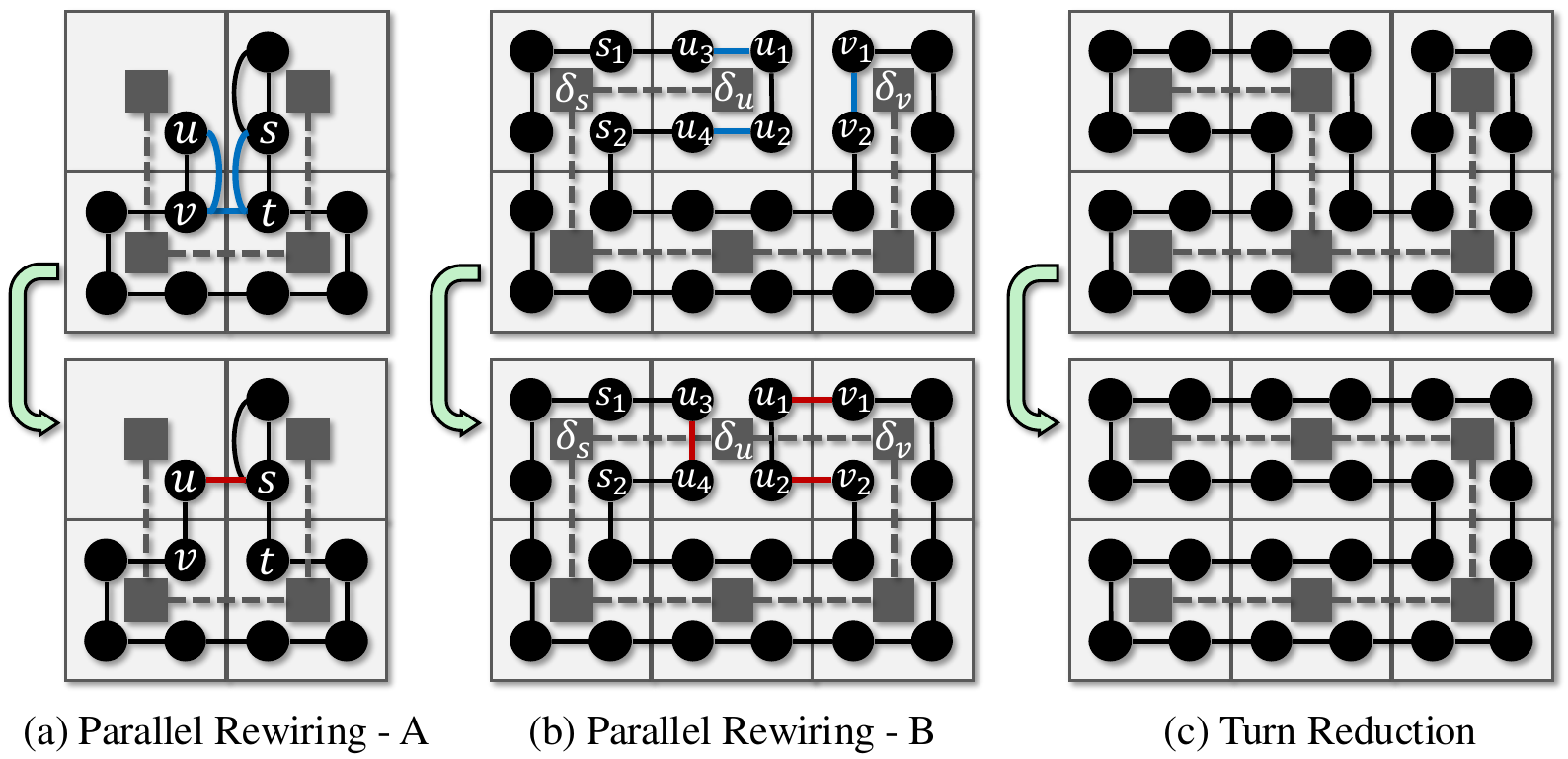}
\caption{ESTC coverage path before (top row) and after local optimizations (bottom row). (a)(b) The two types of parallel rewiring. (c) Turn reduction with horizontally aligned hyperedges.}
\label{fig:local_opt}
\end{figure}

\subsection{Integration into Existing MCPP Methods}\label{subsec:apply_estc}
Existing grid-based MCPP methods often focus on special cases of MCPP that assume complete hypervertices and vertex-weighted graphs. However, our more generally formulated MCPP, as defined in Problem~\ref{problem:_mcpp}, addresses a broader scope that remains relatively unexplored in the literature, despite its significant practical relevance. The proposed ESTC paradigm can be seamlessly integrated into these existing methods, allowing them to effectively solve MCPP even without the above constraints typically assumed.
To demonstrate how existing MCPP methods can be adapted using the ESTC paradigm to solve an MCPP instance on any grid graph $G$, even when some hypervertices are incomplete, we use the following examples categorized into multi-tree and single-tree methods.

\begin{figure*}[t]
\centering
\includegraphics[width=\linewidth]{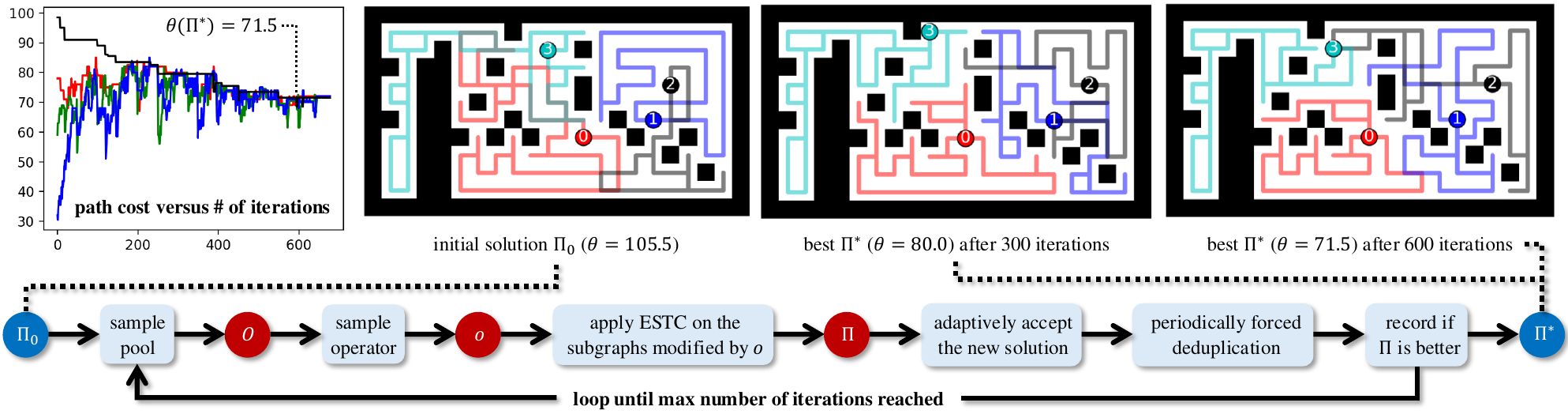}
\caption{The flowchart of our LS-MCPP algorithmic framework for instance \textit{floor-small}. 
Given an initial MCPP solution $\Pi_0$, LS-MCPP iteratively samples a pool $O$ from three pools of dedicated operators and then samples an operator $o$ from $O$, to transform a set $\mathcal{G}$ of subgraphs from one to another. 
It then applies ESTC on the subgraphs to evaluate the new MCPP solution $\Pi$, which is adaptively accepted according to its makespan $\theta$. 
A periodic subroutine of forced deduplication is called at the end of each iteration, and the best solution $\Pi^*$ is recorded if $\Pi$ is better with a lower makespan.}
\label{fig:flowchart}
\end{figure*}

\subsubsection{\textbf{Multi-Tree Methods}}
For MCPP instance $(G,I,R)$, we first generate a set $\{G_i=(V_i,E_i)\}_{i=1}^k$ of $k$ connected subgraphs of $G$. Each $G_i$ contains the root vertex $r_i$. We then solve each sub-CPP instance for robot $i$ on $G_i$ as described in Alg.~\ref{alg:estc} to generate its coverage path.
The following theorem shows a sufficient condition for complete coverage.
\begin{theorem}
If each $G_i$ is connected and $\bigcup_{i=1}^k V_i=V$, then using ESTC is guaranteed to achieve complete coverage of $G$.
\end{theorem}
\begin{proof}
Given that $G_i$ is connected, ESTC is guaranteed to generate a path that covers all its vertices due to Theorem~\ref{theo:complete}. If $\bigcup_{i=1}^k V_i=V$, then the coverage paths generated for all $G_i$ jointly cover all vertices of $G$.
\end{proof}
The design of our LS-MCPP framework is based on the principle expounded in the above theorem. It aims to search for a good set of subgraphs $G_i$ for the robots, each being connected and containing the root vertex $r_i$, with the property that the union of all their vertices comprises the entire vertex set of $G$ for complete coverage.
It now suffices to outline the details below for adapting existing methods to compute the $k$ subgraphs, with comprehensive numerical results in Sec.~\ref{sec:res}.

\noindent\textbf{VOR:}
VOR is a baseline method introduced in~\cite{tang2024large} that is inherently compatible with ESTC.
It partitions $G$ into $k$ subgraphs $\{G_i\}_{i=1}^k$ as a Voronoi diagram~\cite{aurenhammer2000voronoi}, using the set $R$ of root vertices as the pivot points and shortest path distances on $G$ as the distance metric.

\noindent\textbf{MFC:}
Multirobot Forest Coverage (MFC)~\cite{zheng2007robot,zheng2010multirobot} employs the Rooted-Tree-Cover procedure~\cite{even2004min} to heuristically compute a \textit{rooted tree cover} on the quadrant coarsened grid $\mathcal{H}$---a forest of $k$ trees $\{T_i\}_{i=1}^k$, each rooted at the hypervertex containing the root vertex $r_i$, which jointly cover all hypervertices of $\mathcal{H}$. MFC then generates a coverage path for each robot $i$ by applying STC on its tree $T_i$. To integrate MFC with ESTC, we apply the Rooted-Tree-Cover procedure on the hypergraph $H$ (constructed as per Sec.~\ref{subsec:estc}) instead of $\mathcal{H}$ to generate a rooted tree cover on $H$ that induces the set of $k$ subgraphs $\{G_i\}_{i=1}^k$ of $G$.

\noindent\textbf{MIP:}
MIP~\cite{tang2023mixed} solves a mixed-integer program to compute an optimal rooted tree cover on the quadrant coarsened grid $\mathcal{H}$ that minimizes the maximum tree weight. To integrate MIP with ESTC, we adapt the mixed-integer program to compute an optimal rooted tree cover on the hypergraph $H$ (constructed as per Sec.~\ref{subsec:estc}) instead of $\mathcal{H}$, which induces the set of $k$ subgraphs $\{G_i\}_{i=1}^k$ of $G$.

\subsubsection{\textbf{Single-Tree Methods}}
Single-tree methods, unlike multi-tree methods, apply the STC paradigm to only a single tree: STC first computes an MST on the quadrant coarsened grid $\mathcal{H}$ and subsequently generates a coverage path that circumnavigates the MST. The path is then partitioned into $k$ segments, one for each robot. MSTC~\cite{hazon2005redundancy} partitions the path such that each segment starts at a unique root vertex.
Its improvement MSTC$^*$~\cite{tang2021mstc} greedily identifies balanced partitioning points along the path to potentially reduce the makespan and then connects each root vertex to its respective segment via shortest paths.
We can adapt MSTC$^*$ to solve MCPP for scenarios with incomplete hypervertices by simply replacing STC with ESTC.

\section{The LS-MCPP Framework}\label{sec:mcpp}
In this section, we introduce LS-MCPP, our local search algorithmic framework designed to effectively address MCPP as formulated in Problem~\ref{problem:_mcpp}. This framework integrates three types of neighborhood operators and utilizes ESTC as a subroutine, enhancing its ability to solve MCPP even in complex environments with incomplete hypervertices. LS-MCPP aims to optimize coverage paths by iteratively refining subgraphs and their corresponding coverage paths, ensuring both efficiency and completeness in coverage solutions.

\subsection{Overview}
As shown in Fig.~(\ref{fig:flowchart}), LS-MCPP begins with an initial solution for the MCPP instance $(G, I, R)$ with its corresponding set $\mathcal{G}_0=\{G_i=(V_i,E_i)\}_{i=1}^{k}$ of $k$ connected subgraphs induced by the solution paths.
LS-MCPP operates as an iterative procedure that employs a two-layer hierarchical sampling scheme, progressively adjusting the subgraph set $\mathcal{G}$ to efficiently explore the constructed neighborhood: The first layer selects an operator pool using \textit{roulette wheel} selection~\cite{goldberg1989genetic} from three operator pools, each containing a different type of operators; The second layer heuristically selects an operator from the chosen pool, as detailed in Sec.~\ref{subsec:operator_sampling}.

We categorize $G_i$ as a \textit{light subgraph} if its coverage path cost $c(\pi_i)$ from ESTC is no larger than the average coverage path cost $\bar{c}$ over all subgraphs; otherwise, it is a \textit{heavy subgraph}.  We define the \textit{duplication set} $V^+=\{v\in V\,|\,n_v>1 \}$ as the set of vertices included in more than one subgraph, where $n_v=\sum_{i=1}^k\left|\{x\in V_{i}|x=v\}\right|$ counts the occurrences of vertex $v\in V$ in all subgraphs.

LS-MCPP employs \textit{grow operators} on light subgraphs to assign them new vertices to cover, \textit{deduplicate operators} on heavy subgraphs to eliminate unnecessary duplication, and \textit{exchange operators} to balance path costs between subgraphs with a large cost difference.
The operators are detailed in Sec.~\ref{subsec:operators}, designed to attain a cost-equilibrium MCPP solution with a low makespan.
LS-MCPP then employs ESTC on each subgraph $G_i$ to generate its coverage path and evaluate the makespan. 
LS-MCPP also periodically calls a deduplication function to exploit the current subgraphs to achieve a low-makespan solution.

\subsection{Algorithm Description}\label{subsec:lsmcpp-overview}

\begin{algorithm}[tb]
\DontPrintSemicolon
\linespread{0.95}\selectfont
\caption{LS-MCPP}\label{alg:ls-mcpp}
\SetKwInput{KwInput}{Input}
\SetKwInput{KwParam}{Param}
\KwInput{MCPP instance $(G, I, R)$, initial solution $\Pi_0$}
\KwParam{max iteration $M\in\mathbb{Z}^+$, forced deduplication step $S\in\mathbb{Z}^+$, temperature decay factor $\alpha\in[0,1]$, pool weight decay factor $\gamma\in[0, 1]$}
$\mathcal{G}=\{G_i\}_{i=1}^k\gets$ set of subgraphs extracted from $\Pi_0$\;\label{alg:ls:init_G_0}
$\Pi\gets\Pi_0$, $\Pi^*\gets\Pi_0$ \Comment{current and optimal solutions}\;
$\mathcal{O}\gets\{\emptyset, \emptyset, \emptyset\}$\Comment{set of pools of operators}\;\label{alg:ls:create_O}
\textit{Update-Pools}\,($\mathcal{G}$, $\mathcal{O}$, $\Pi$)\Comment{initialize the pools}\;\label{alg:ls:init_O}
$t\gets 1$ \Comment{temperature in \textit{simulated annealing}}\;\label{alg:ls:init_t}
$\mathbf{p}\gets [1, 1, 1]$ \scalebox{0.96}{\Comment{pool weights for \textit{roulette wheel} selection}}\;\label{alg:ls:init_p}
\For{$i\gets 1, 2,\ldots , M$}
{\label{alg:ls:main_iter}
    $O\sim\mathcal{O}$ \Comment{sample a pool by the probability $\sigma(\mathbf{p})$}\;\label{alg:ls:rw_sampling}
    update the weight $\mathbf{p}[O]$ of the selected pool $O$ to $(1-\gamma)\cdot\mathbf{p}[O]+\gamma\cdot\max (-\Delta\theta, 0)$\;\label{alg:ls:update_pool_weight}
    $\mathbf{h}\gets$ stacked heuristic values for all operators in $O$ \;\label{alg:ls:heur_vec}
    $o\sim O$ \Comment{sample an operator by probability $\sigma(\mathbf{h})$}\;\label{alg:ls:op_sampling}
    $\mathcal{G}'\gets$ apply $o$ to its relating subgraph(s) of $\mathcal{G}$\;\label{alg:ls:calc_G_prime}
    $\Pi'\gets$ apply ESTC on updated subgraph(s) in $\mathcal{G}'$\;\label{alg:ls:ESTC}
    \eIf{$\Delta\theta\gets \theta(\Pi') - \theta(\Pi)<0$}{\label{alg:ls:ada_acc_st}
        update $\mathcal{G}$ to $\mathcal{G}'$ and $\Pi$ to $\Pi'$\;
    }
    {
        update $\mathcal{G}$ to $\mathcal{G}'$ and $\Pi$ to $\Pi'$ with a probability of $\exp{(\frac{-\Delta\theta}{t})}$\;\label{alg:ls:ada_acc_ed}
    }
    call \textit{Update-Pools}\,($\mathcal{G}$, $\mathcal{O}$, $\Pi$, $o$) if $\mathcal{G}$ is updated\;\label{alg:ls:update_pools}
    \If{$i\,\%\, S = 0 \;\,\text{or}\;\, \Delta\theta<0$}
    {
        call \textit{Forced-Deduplication}\,($\mathcal{G}$, $\mathcal{O}$, $\Pi$) \;\label{alg:ls:force_dedup}
    }
    update $\Pi^*$ to $\Pi$ if $\theta(\Pi) < \theta(\Pi^*)$ \;\label{alg:ls:record_opt}
    update the temperature $t$ to $\alpha\cdot t$\;\label{alg:ls:update_temp}
}
\Return improved MCPP solution $\Pi^*$ \;\label{alg:ls:return}

\end{algorithm}

\noindent\textbf{Pseudocode (Alg.~\ref{alg:ls-mcpp}):}
LS-MCPP begins by computing the corresponding set $\mathcal{G}_0$ of subgraphs on the input initial MCPP solution $\Pi_0$ [Line~\ref{alg:ls:init_G_0}].
It then initializes a set of three operator pools via a subroutine \textit{Update-Pools} described in Sec.~\ref{subsec:operators}, containing only grow operators, deduplicate operators, and exchange operators, respectively [Lines~\ref{alg:ls:create_O}-\ref{alg:ls:init_O}].
It also initializes the temperature scalar $t$ and the pool weight vector $\mathbf{p}$ [Lines~\ref{alg:ls:init_t}-\ref{alg:ls:init_p}].
The scalar $t$ is scaled down by the decay factor $\alpha\in[0,1]$ for every iteration [Line~\ref{alg:ls:update_temp}] to dynamically adjust the probability for LS-MCPP to become increasingly less likely to accept a non-improving solution over iterations, which is the \textit{simulated annealing} strategy~\cite{van1987simulated} to skip local minima. The vector $\mathbf{p}$ represents the pool weight for each of the three aforementioned pools.
During each iteration [Line~\ref{alg:ls:main_iter}], vector $\mathbf{p}$ is used for roulette wheel selection~\cite{goldberg1989genetic} to select a pool $O$ from $\mathcal{O}$ [Line~\ref{alg:ls:rw_sampling}], where $\sigma(\mathbf{p})$ is the \textit{softmax} function of $\mathbf{p}$  determining the probability of selecting each pool. 
Once the pool $O$ is selected, its corresponding weight $\mathbf{p}[O]$ is updated with a weight decay factor $\gamma\in[0,1]$ [Line~\ref{alg:ls:update_pool_weight}].
Similar to the pool selection, LS-MCPP samples an operator $o$ from the selected pool $O$ with $\sigma(\mathbf{h})$ as the probability function [Line~\ref{alg:ls:op_sampling}]. 
Here, $\mathbf{h}$ is the vector of heuristic function $h$ evaluated on all operators in $O$ [Line~\ref{alg:ls:heur_vec}], detailed in Sec.~\ref{subsec:operator_sampling}. 
Following this, LS-MCPP applies operator $o$ to its relating subgraph(s) to obtain a new set $\mathcal{G}'$ of subgraphs and a new solution $\Pi'$ [Lines~\ref{alg:ls:calc_G_prime}-\ref{alg:ls:ESTC}] and calculates the makespan increment $\Delta\theta$ to determine whether to accept the new solution $\Pi'$.
LS-MCPP accepts the new solution $\Pi'$ if it has a smaller makespan, and, otherwise, accepts it with a probability of $\exp{(-\Delta\theta/t)}\in[0,1]$ [Lines~\ref{alg:ls:ada_acc_st}-\ref{alg:ls:ada_acc_ed}]. 
Upon the acceptance of $\Pi'$, LS-MCPP updates the set $\mathcal{O}$ of operator pools via \textit{Update-Pools} (Alg.~\ref{alg:up}) [Line~\ref{alg:ls:update_pools}].
LS-MCPP also calls a subroutine \textit{Forced-Deduplication} (Alg.~\ref{alg:fd}) every $S$ iteration or the makespan has been reduced in the current iteration [Line~\ref{alg:ls:force_dedup}].
The iterative search procedure terminates when it reaches the maximum number of iterations $M$, and LS-MCPP returns the best solution $\Pi^*$ found [Line~\ref{alg:ls:return}].

\begin{algorithm}[tb]
\DontPrintSemicolon
\linespread{0.95}\selectfont
\caption{Forced-Deduplication}\label{alg:fd}
\SetKwInput{KwInput}{Input}
\KwInput{set $\mathcal{G}=\{G_i\}_{i=1}^k$ of subgraphs,\quad\quad\quad\quad\quad\quad 
set $\mathcal{O}=\{O_g, O_d, O_e\}$ of pools,\quad\quad\quad\quad\quad\quad 
set $\Pi=\{\pi_i\}_{i=1}^k$ of coverage paths}

\For{$i\in I$ sorted by descending $c(\pi_i)$}
{\label{func:forcededup:main_loop_1}
    \While{$(u, v)\gets$ any U-turn in $\pi_i$}
        {\label{func:forcededup:find_U_turn}
            remove $u, v$ from $G_i$ and $\pi_i$ \;\label{func:forcededup:update}
        }
}

\For{$i\in I$ sorted by descending $c(\pi_i)$}
{\label{func:forcededup:main_loop_2}
    \For{$o\in$ $O_d$ sorted by descending $h(o)$\label{func:forcededup:order}}
    {
        apply $o$ on $G_i$ and update $\pi_i$ via ESTC\;\label{func:forcededup:update2}
    }
}
call \textit{Update-Pools}\,($\mathcal{G}$, $\mathcal{O}$, $\Pi$)\;\label{alg:fd:update_pools}
\end{algorithm}

\noindent\textbf{Forced Deduplication:}
A critical aspect contributing to the success of LS-MCPP is restricting duplication to only those necessary over the neighborhood exploration, such as those in a narrow path connecting two separate regions of graph $G$. 
For this purpose, LS-MCPP uses the \textit{Forced-Deduplication} subroutine as described in Alg.~\ref{alg:fd} to deduplicate all possible duplication in two folds.
The first part iterates through each path $\pi_i\in\Pi$ in a cost-decreasing order to remove any U-turns [Lines~\ref{func:forcededup:main_loop_1}-\ref{func:forcededup:update}].
A U-turn of $\pi_i$ is defined as an edge traversal $(u\rightarrow v)$ in a ``U-turn'' pattern segment $p\rightarrow u\rightarrow v\rightarrow q$ in path $\pi_i$.
If $u, v\in V_{i}\cap V^+$, then we can safely remove $(u\rightarrow v)$ to reduce $c(\pi)$ without losing any vertex coverage or introducing any new duplication.
Fig.~(\ref{fig:U_turn}) illustrates the procedure to remove U-turns for a path segment progressively. A special case of U-turn often occurs for incomplete hypervertices where the ``U-turn'' pattern becomes $p=v\rightarrow u\rightarrow v\rightarrow q$.
The second part [Lines~\ref{func:forcededup:main_loop_2}-\ref{func:forcededup:update2}] iteratively process each graph $G_i$ in descending order of path cost $\pi_i$ [Line~\ref{func:forcededup:main_loop_2}], applying each deduplicate operator $o$ from $O_d$ in descending order of its heuristic value $h(o)$ [Line~\ref{func:forcededup:order}].
After the above two parts, the function re-initializes all the pools in $\mathcal{O}$ by calling \textit{Update-Pools} (Alg.~\ref{alg:up}) [Line~\ref{alg:fd:update_pools}].
It is worth noting that the first part can remove edges that do not constitute deduplicate operators, and some deduplicate operators can remove edges that are not U-turns depending on the MST generated by ESTC. 
This complementary relation between the two parts contributes to the improved performance of LS-MCPP.

\begin{figure}[t]
\centering
\includegraphics[width=0.8\columnwidth]{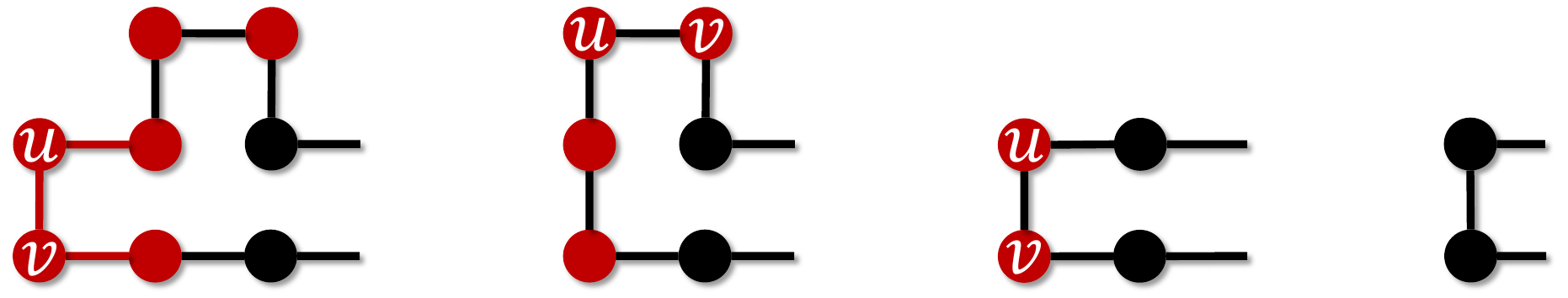}
\caption{Subroutine \textit{Forced-Deduplication} executed on a segment of coverage path $\pi_i$, where red and black circles represent vertices from $V_{i}\cap V^+$ and $V_i\setminus V^+$, respectively. Frames from left to right: a U-turn $(u, v)\in E(\pi_i)$ is found, and $u, v$ are removed from $\pi_i$ until no U-turn exists.}
\label{fig:U_turn}
\end{figure}

\noindent\textbf{Initial Solution:}
For the choice of the initial solution $\Pi_0$ in LS-MCPP, we limit ourselves to existing multi-tree MCPP methods that can generate high-quality solutions efficiently. 
Specifically, we choose our initial solution from VOR or MFC, depending on which provides the better solution quality.
Both methods are integrated with the proposed ESTC paradigm, as described in Sec.~\ref{subsec:apply_estc}.

\begin{remark}
Our LS-MCPP framework is inherently better suited to utilizing a multi-tree method for generating an initial MCPP solution, which aligns with its design to iteratively improve a set of $k$ subgraphs using a local search strategy and apply ESTC to obtain coverage paths for evaluating the solution cost in each iteration. While a single-tree method can also provide initial subgraphs for LS-MCPP, induced by its solution coverage paths, our empirical finding~\cite{tang2024large} indicates that this approach typically results in significant overlaps among subgraphs and tends to reduce the effectiveness of the neighborhood operators within our LS-MCPP framework.
\end{remark}

\subsection{Pools of Operators}\label{subsec:operators}

We introduce three types of operators: grow operators, deduplicate operators, and exchange operators.
These operators are designed to modify the boundaries of each connected subgraph $G_i=(V_i, E_i)\in\mathcal{G}$ of the input graph $G=(V,E)$.
An operator alters the set $\mathcal{G}$ while ensuring that its property of complete coverage (i.e., $\bigcup_{i=1}^k V_{i}=V$) and subgraph connectivity remains invariant. 
Each operator aims to introduce only essential duplication, with the potential to enhance solution quality.

\begin{figure}[t]
\centering
\includegraphics[width=\columnwidth]{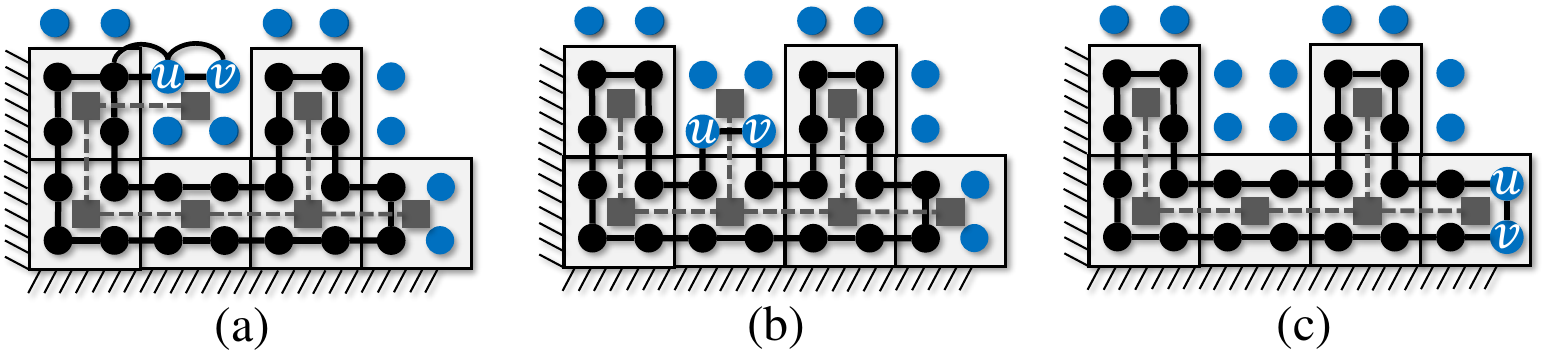}
\caption{ESTC path after applying an edge-wise $o_g^{\text{(e)}}(i,\{u,v\})$ on subgraph $G_i=(V_i,E_i)$. Blue circles and black circles represent vertices of the boundary vertex set $B_i$ and vertices of $V_{i}$, respectively. (a) An invalid operator without any parallel edge in $G_i$. (b)(c) Two valid edge-wise grow operators.}
\label{fig:grow_op}
\end{figure}

Specifically, we consider both \textit{edge-wise} and \textit{vertex-wise} operations for these operators.
An edge-wise operator alters a single edge at a time, whereas a vertex-wise operator alters a single vertex at a time.
Altering a crossing edge $(u,v)\in E$ with $\delta_u\neq\delta_v$ often introduces unnecessary vertex duplications.
Therefore, we restrict the edge-wise operators to only altering intra-hypervertex edges $E^\text{intra}=\{(u,v)\in E\,|\,\delta_u=\delta_v\}$.
In practice, edge-wise operators generally perform well but sometimes cannot be constructed in configurations with many incomplete hypervertices, which may lead to early termination of the local search (Alg.~\ref{alg:ls-mcpp}).
Vertex-wise operators are less restricted than edge-wise operators and tend to introduce unnecessary vertex duplications if applied to complete hypervertices.
Therefore, vertex-wise operators serve only as complementary tools when no edge-wise operators are available, as reflected in the \textit{Update-Pools} function (Alg.~\ref{alg:up}), which initializes or updates the operator pools.\footnote{Our preliminary conference version~\cite{tang2024large} uses only edge-wise operators, which may lead to early termination issues. See Sec.~\ref{subsec:ablation} for a detailed ablation study on the impact of vertex-wise operators.}

\noindent\textbf{Grow Operator:}
A grow operator serves to expand a subgraph to cover additional vertices already covered by other subgraphs.
Let $B_i$ denote the set of ``boundary'' vertices that are not part of $G_i$ but adjacent to a vertex of $G_i$, defined as $B_i=\{v\in V\setminus V_{i}\,|\,\exists(u,v)\in E, u\in V_i\}$.
A vertex-wise grow operator $o_g^{\text{(v)}}(i,\{v\})$ adds vertex $v\in B_i$ and all its adjacent edges to subgraph $G_i$.
An edge-wise grow operator $o_g^{\text{(e)}}(i,\{u,v\})$ adds edge $e=(u,v)\in E_i\cap E^\text{intra}$ with $u,v\in B_i$ and all its adjacent edges into $G_i$, where $\exists (p, q)\in E_{i}$ such that $(u,p), (v,q)\in E_{i}$. 
In essence, a valid $o_g(i,e)$ can only add an edge $(u, v)$ if there exists a parallel edge $(p, q)$ in $G_i$, as shown in the two examples in Fig.~(\ref{fig:grow_op}b) and (\ref{fig:grow_op}c).
This design choice avoids introducing undesirable routing to the coverage path, as shown in the example in Fig.~(\ref{fig:grow_op}a).
After its execution, a grow operator $o_g^{\text{(v)}}(i,\{v\})$ (or $o_g^{\text{(e)}}(i,\{u,v\})$) effectively adds vertex $v$ (or vertices $u$ and $v$) to $V^+$.

\noindent\textbf{Deduplicate Operator:}
A deduplicate operator serves to remove unnecessary duplication from a subgraph.
A vertex-wise deduplicate operator $o_d^{\text{(v)}}(i,\{v\})$ removes vertex $v\in V_i\cap V^+$ and all its adjacent edges from $G_i$ such that $G_i$ remains connected.
For an edge $e=(u, v)\in E_i\cap E^\text{intra}$, let $\delta^{\,t}_{e}$, $\delta^{\,b}_{e}$, $\delta^{\,l}_{e}$, and $\delta^{\,r}_{e}$ denote the four neighbors of $\delta_u$ (also equal to $\delta_v$) based on the positioning of $u$ and $v$.
Fig.~(\ref{fig:dedup_op}a) demonstrates one example that can be rotated to various symmetric cases (with only $\delta^{\,t}_{e}$ potentially containing both a vertex adjacent to $u$ and one adjacent to $v$).
An edge-wise deduplicate operator $o_d^{\text{(e)}}(i,\{u,v\})$ removes edge $(u,v)\in E_i\cap E^\text{intra}$ with $u,v\in V_i\cap V^+$ and all its adjacent edges from $G_i$ such that $G_i$ remains connected. If $\delta_u$ contains only $u$ and $v$, then this removal is always valid and reduces the cost of the corresponding ESTC path of $G_i$, as shown in Fig.~(\ref{fig:dedup_op}c). Otherwise, this removal is constrained by three conditions: (1) any vertex contracted to $\delta_{e}^t$ is not in $V_{i}$; (2) all the four vertices contracted to $\delta^{b}_{e}$ are in $V_{i}$; and (3) if any $v\in V_i$ is contracted to $\delta_{e}^l$ (or $\delta_{e}^r$), then both $\delta_{e}^l$ (or $\delta_{e}^r$) and its (only) common neighboring hypervertex with $\delta_{e}^b$ must be complete, and all their vertices must also be in $V_{i}$. Such validity filtering on the edge-wise deduplicate operators is essential to avoid introducing unnecessary duplications to the corresponding ESTC path of $G_i$, as shown in Fig.~(\ref{fig:dedup_op}).
After its execution, a duplicate operator $o_d^{\text{(v)}}(i,\{v\})$ (or $o_d^{\text{(e)}}(i,\{u,v\})$) effectively removes vertex $v$ (or vertices $u$ and $v$) from $V^+$ if the vertex (or vertices) is not included in multiple subgraphs.

\noindent\textbf{Exchange Operator:}
An exchange operator combines a grow operator and a deduplicate operator so that the deduplicate operator immediately removes the new duplication introduced by the grow operator, potentially improving the effectiveness of a single operator and speeding up the convergence of the local search.
A vertex-wise exchange operator $o_e^{\text{(v)}}(i,j,\{v\})$ combines $o_g^{\text{(v)}}(i,\{v\})$ and $o_d^{\text{(v)}}(j,\{v\})$, where $v\in B_i\cap V_j$.
An edge-wise exchange operator $o_e^{\text{(e)}}(i,j,\{u,v\})$ combines $o_g^{\text{(e)}}(i,\{u,v\})$ and $o_d^{\text{(e)}}(j,\{u,v\})$. Note that the operators $o_g$ and $o_d$ being combined must be both valid except that vertex $v$ for a vertex-wise $o_d^{\text{(v)}}(j,\{v\})$ and vertices $u$ and $v$ for an edge-wise $o_d^{\text{(e)}}(j,\{u,v\})$ are not required to be in $V^+$ since they are not necessarily covered by subgraphs other than $G_j$ before the grow operator is applied.
After the execution of an exchange operator, $V^+$ remains unchanged.

\begin{figure}[t]
\centering
\includegraphics[width=\columnwidth]{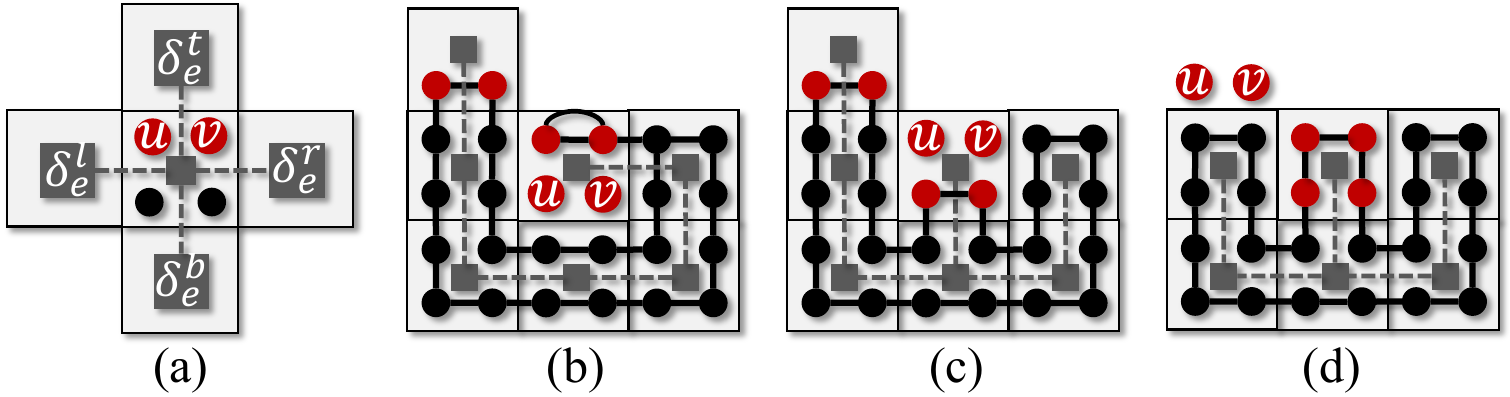}
\caption{ESTC path after applying an edge-wise $o_d^{\text{(e)}}(i,\{u,v\})$ on subgraph $G_i=(V_i,E_i)$. Red circles and black circles represent vertices in $V_i\cap V^+$ and not in $V_i\setminus V^+$, respectively. (a) Four neighbors of the hypervertex $\delta_u(=\delta_v$). (b) An invalid operator violating the third condition when $\delta_u$ is complete. (c)(d) Two valid edge-wise deduplicate operators.}
\label{fig:dedup_op}
\end{figure}


\noindent\textbf{Update Pools:}
The \textit{Update-Pools} function in Alg.~\ref{alg:up} initializes the set $\mathcal{O}$ of pools of operators [Line~\ref{alg:up:init_no_op}] or updates the set $\mathcal{O}$ whenever $\mathcal{G}$ is changed by some operator $o$ [Lines~\ref{alg:up:init_w_op_1}-\ref{alg:up:init_w_op_2}].
Given the set $\Delta V$ of dirty vertices (defined in line~\ref{alg:up:init_w_op_2} or line~\ref{alg:up:init_no_op}), \textit{Update-Pools} first removes all operators involving any $v\in\Delta V$ from each pool [Line~\ref{alg:up:rmv_all}].
For each $v\in\Delta V$, it then iterates each light subgraph to add a valid edge-wise grow operator to $O_g$ if one exists; otherwise, it adds a vertex-wise grow operator [Lines~\ref{alg:up:for_all_Og}-\ref{alg:up:add_vw_Og}].
Similarly, deduplicate operators are added to $O_d$ by iterating each heavy subgraph [Lines~\ref{alg:up:for_all_Od}-\ref{alg:up:add_vw_Od}].
Finally, the pool $O_e$ of exchange operators is updated by combining the newly added operators in $O_g$ and $O_d$ [Line~\ref{alg:up:init_Oe}].

\begin{algorithm}[t]
\DontPrintSemicolon
\linespread{0.95}\selectfont
\caption{Update-Pools}\label{alg:up}
\SetKwInput{KwInput}{Input}
\KwInput{set $\mathcal{G}=\{G_i\}_{i=1}^k$ of subgraphs,\quad\quad\quad\quad\quad\quad 
set $\mathcal{O}=\{O_g, O_d, O_e\}$ of pools,\quad\quad\quad\quad\quad\quad 
set $\Pi=\{\pi_i\}_{i=1}^k$ of coverage paths,\quad\quad\quad\quad\quad\quad 
[\textit{optional}] most recent applied operator $o$}
\eIf{$o$ is provided as an input}{
    $V_o\gets$ the set of involved vertices in $o$\;\label{alg:up:init_w_op_1}
    $\Delta V\gets V_o\cup\{u\in V\,|\, v\in V_o, (u,v)\in E\}$\;\label{alg:up:init_w_op_2}
}{
$\Delta V\gets V$\Comment{update for all vertices in $G=(V,E)$}\;\label{alg:up:init_no_op}
}
\scalebox{0.93}{remove operators from each $O\in\mathcal{O}$ involving any $v\in\Delta V$}\;\label{alg:up:rmv_all}
\For{$v\in\Delta V$}{
\For{$i\in \{j\in I\,|\, c(\pi_j)\leq\bar{c}\}$}{\label{alg:up:for_all_Og}
    \eIf{found a valid $o_g^{\text{(e)}}=(i,\{u,v\})$}{\label{alg:up:check_ew_Og}
        $O_g\gets O_g\cup\{o_g^{\text{(e)}}\}$\label{alg:up:add_ew_Og}
    }{
        $O_g\gets O_g\cup\{o_g^{\text{(v)}}=(i,\{v\})\}$\label{alg:up:add_vw_Og}
    }
}
\For{$i\in \{j\in I\,|\, c(\pi_j)>\bar{c}\}$}{\label{alg:up:for_all_Od}
    \eIf{found a valid $o_d^{\text{(e)}}=(i,\{u,v\})$}{\label{alg:up:check_ew_Od}
        $O_d\gets O_d\cup\{o_d^{\text{(e)}}\}$\label{alg:up:add_ew_Od}
    }{
        $O_d\gets O_d\cup\{o_d^{\text{(v)}}=(i,\{v\})\}$\label{alg:up:add_vw_Od}
    }
}
}
\scalebox{0.97}{update $O_e$ by iterating newly added operators in $O_g, O_d$}\;\label{alg:up:init_Oe}
\end{algorithm}

\noindent\textbf{Properties:} After any valid operator is executed, (1) every $G_i\in\mathcal{G}$ remains connected since only deduplicate operators can remove edges but a valid deduplicate operator guarantees that $G_i$ remains connected, and (2) the property of the union of the vertices of all subgraphs being the vertex set $V$ is not affected since no vertex coverage is lost.
Therefore, an LS-MCPP solution ensures complete coverage due to Theorem~\ref{theo:complete}.

\subsection{Heuristics for Operator Sampling}\label{subsec:operator_sampling}
To better explore the solution neighborhood, LS-MCPP should carefully determine which operator to sample from the selected pool (line~\ref{alg:ls:op_sampling} of Alg.~\ref{alg:ls-mcpp}).
We propose three heuristic functions tailored to the three types of operators to evaluate their potential in guiding the neighborhood search and improving the solution. 
An operator with a larger heuristic value results in a higher probability of being sampled.
The heuristic function for grow operators has two considerations: (1) prioritizing growing light subgraphs with smaller path costs and (2) prioritizing covering vertices with less duplication. 
Formally, the heuristic value $h(o_g)$ for a grow operator $o_g(i,V_{o_g})$ is defined as $h(o_g) = -k\cdot c(\pi_i) - \frac{1}{|V_{o_g}|}\cdot\sum_{v\in V_{o_g}}n_v$. Note that $k \geq n_v$ holds for any $v\in V$, thereby prioritizing consideration (1) over consideration (2).
The heuristic function for deduplicate operators $o_d(i,V_{o_d})$ is defined as the exact opposite of the one for grow operators: $h(o_d) = k\cdot c(\pi_i) + \frac{1}{|V_{o_d}|}\cdot\sum_{v\in V_{o_d}}n_v$.
The heuristic function for exchange operators $o_e(i,j,V_{o_e})$ is defined as $h(o_e) = c(\pi_j) - c(\pi_i)$ to prioritize pairs of subgraphs with larger differences in their path costs.

\section{Deconflicting MCPP Solutions}\label{sec:deconf}
In this section, we address inter-robot conflicts in coverage paths produced by LS-MCPP or other grid-based MCPP methods. Our MCPP formulation (Problem~\ref{problem:_mcpp}) does not inherently prevent collisions among robots. In an MCPP solution with a small makespan, robots typically cover different regions and are distributed across the grid, which lowers the likelihood of conflicts. However, certain regions remain prone to congestion and conflicts, particularly around clustered root vertices or narrow corridors that only allow one robot to pass at a time. 

To resolve these potential conflicts, we formulate the deconflicting task as a variant of MAPF, leveraging MAPF techniques to ensure each robot follows its designated coverage path while avoiding collisions with others. Specifically, we develop a two-level hierarchical planner, following standard MAPF practices, to effectively deconflict MCPP solutions: The high-level planner employs Priority-Based Search (PBS)~\cite{ma2019searching} to assign priorities among pairs of conflicting robots; The low-level planner is a novel algorithm that plans a continuous-time path for each robot that avoids conflicting with any higher-priority robots.
Our key extensions from standard MAPF lie in 1) our treatment of non-uniform edge weights, which translate to variable time costs for robot actions and 2) an adaptive approach of the low-level planner for time-efficient single-robot path planning with multiple goals.

\subsection{Formulation of Deconflicted MCPP}~\label{subsec:deconf-pre}
Given an input solution $\Pi=\{\pi_i\}_{i=1}^k$ for the MCPP instance $(G,I,R)$, we define the \textit{Deconflicted MCPP} problem, which aims to construct a set $\mathcal{T}=\{\tau_i\}_{i=1}^k$ of $k$ conflict-free trajectories.
A \textit{trajectory} $\tau_i$ for robot $i$ is an ordered sequence $(\mathbf{x}_1, \ldots, \mathbf{x}_{|\tau_i|})$ of time-embedded \textit{states} $\mathbf{x}=(v, t)$ indicating that the robot stays at vertex $v\in V$ at time $t\geq 0$. 
Given path $\pi_i=(v_1,\ldots,v_{|\pi_i|})$ where $v_1=v_{|\pi_i|}=r_i\in R$, trajectory $\tau_i$ must start from state $\mathbf{x}_1=(v_1, 0)$, visit all vertices in the specified order of $\pi_i$, and end with state $\mathbf{x}_{|\tau_i|}=(v_{|\pi_i|}, \cdot)$. Formally, there exist indices $a_1 = 1< a_2 <\ldots< a_{|\pi_i|} = |\tau_i|$ such that $\mathbf{x}_{a_j}=(v_j,\cdot)$ for all $j=1,\ldots,|\pi_i|$.
To transit from state $(u, t)$ to the next state $(v, t')$ in a trajectory, the robot performs a \textit{wait-and-move} action, where it waits at $v$ for $t'-w_e - t \geq 0$ and then moves from $u$ to $v$ along edge $e=(u,v)\in E$ with a time cost of $w_{e}>0$.
Let $\tau[j]$, $\mathbf{x}.v$ and $\mathbf{x}.t$ denote the $j$-th state in $\tau$, the vertex and time in each state $\mathbf{x}$, respectively.
The quality of a Deconflicted MCPP solution $\mathcal{T}$ is evaluated using the \textit{makespan} metric, denoted by $\theta(\mathcal{T})=\max\{\tau_1[|\tau_1|].t,\ldots,\tau_k[|\tau_k|].t\}$, which is consistent with that used for MCPP.

\begin{definition}[Conflict]\label{def:conflict}
A conflict between trajectory $\tau$ with states $\mathbf{x}_j$ for $j=1,\ldots,|\tau|$ and another trajectory $\tau'$ with states $\mathbf{x}'_{j'}$ for $j'=1,\ldots,|\tau'|$ occurs if there exist states $\mathbf{x}_j$ and $\mathbf{x}'_{j'}$ such that their vertices $\mathbf{x}_j.v=\mathbf{x}'_{j'}.v$ and the two time intervals $(\mathbf{x}_{j-1}.t, \mathbf{x}_{j+1}.t)$ and $(\mathbf{x}'_{j'-1}.t, \mathbf{x}'_{j'+1}.t)$ intersect. To ensure conflict checks are valid for all $j$, we assume two dummy bounds $\mathbf{x}_0.t=0$ and $\mathbf{x}_{|\tau|+1}.t=+\infty$. 
\end{definition}
In our definition of conflicts, a vertex is considered occupied by a robot not only when the robot stays at it (as in typical MAPF frameworks~\cite{sharon2015conflict}) but also when a robot is moving into or out of it. This conservative definition subsumes all types of conflicts outlined in MAPF literature~\cite{stern2019multi}, including vertex conflicts, edge conflicts, following conflicts, and swapping conflicts. This generalized definition accommodates continuous time and non-uniform edge weights, unlike traditional MAPF definitions that assume discretized, uniform time steps.


\begin{remark}
Instead of first solving MCPP and then deconflicting the resulting paths as a post-processing procedure, one could integrate the deconfliction at the end of each LS-MCPP iteration (Alg.~\ref{alg:ls-mcpp}), potentially improving performance beyond our proposed sequential approach.
However, such an approach incurs significantly higher runtime due to the computational demands of repeated runs of the deconfliction (see Tab.~\ref{tab:dec_mcpp_running_time}).
\end{remark}

\subsection{The High-Level Planner}
The high-level planner employs Priority-Based Search (PBS)~\cite{ma2019searching} to resolve conflicts defined in Definition~\ref{def:conflict}. PBS searches a priority tree where each node $N$ stores a unique priority ordering given by a set $\boldsymbol{\pmb\prec}_N$ of ordered pairs of robots and a set $N.\mathcal{T}$ of $k$ trajectories that respect the prioritized planning scheme specified by $\boldsymbol{\pmb\prec}_N$.
The root node contains an empty priority set and a set of potentially conflicting trajectories.
When PBS expands a node $N$, it
checks $N.\mathcal{T}$ for conflicts. 
If no conflicts are found, then $N$ is a goal node, and $N.\mathcal{P}$ is the solution. 
If conflicts exist between the trajectories of robots $i$ and $j$, PBS generates two child nodes, $N_1$ and $N_2$, and adds the pair $i\prec j$ (robot $i$ has a higher priority than robot $j$) to $N_1$ and $j\prec i$ to $N_2$.
For child node $N_1$, PBS invokes a low-level planner to update the trajectories of robot $j$ and other robots $j'$ with $j\prec_{N_1} j'$, ensuring $N_1.\mathcal{T}$ respects $\boldsymbol{\pmb\prec}_{N_1}$. 
PBS processes $N_2$ symmetrically.
The low-level planner computes a trajectory for each robot that avoids conflicts with higher-priority trajectories. 
PBS performs Depth-First Search on the priority tree by selecting the child node whose trajectories have a lower makespan in the next iteration, thus exploring the space of all priority orderings.

PBS is incomplete in general; however, with a complete low-level planner, it is complete for well-formed instances of various MAPF problems~\cite{ma2019searching}. Specifically, PBS with a complete low-level planner is complete for well-formed Deconflicting MCPP instances, defined as follows.
\begin{definition}[Well-Formedness]\label{def:wellformness}
A Deconflicted MCPP instance with a given input solution $\Pi$ of MCPP instance $(G,I,R)$ is \textit{well-formed} if, for all $i\in I$, path $\pi_i \in \Pi$ does not traverse any root vertex $r_j\in R$ of robot $j\neq i$.
\end{definition}
Since the given solution $\Pi$ might not satisfy the well-formed condition, before deconflicting the given $\Pi$, we preprocess each $\pi_i\in\Pi$ to remove any $v\in R\setminus\{r_i\}$ from $\pi_i$, which preserves complete coverage while ensuring the completeness of PBS when, for all $i$, $G$ remains connected with all vertices $v\in R\setminus\{r_i\}$ removed. A sufficient condition is that $G$ remains connected with all vertices $v\in R$ removed, which is used to generate instances for our experiments in Sec.~\ref{sec:res}.


\begin{remark}\label{remark:decmcpp_hardness}
Alternatively, Conflict-based Search (CBS)~\cite{sharon2015conflict} can serve as the high-level planner. CBS is optimal for MAPF problems when an optimal low-level planner is used.
However, CBS generally incurs significantly longer runtimes than PBS due to its Best-First Search on a tree whose size grows exponentially with the number of conflicts. 
This runtime increase is even more pronounced in our multi-goal MAPF setting, where an optimal low-level planner is notably more computationally intensive than in standard single-goal MAPF.
According to~\cite{mouratidis2024fools} which applies CBS to OMG-MAPF, a special case of our Deconflicting MCPP problem with unweighted edges, CBS only scales to sequences of up to 5 goals on relatively small graphs. \end{remark}

\subsection{The Low-level Planner}\label{subsec:low-level-planner}
The lower-level computes a trajectory $\tau_i$ for robot $i$ that sequentially visits vertices along the given $\pi_i$. It uses a \textit{reservation table} that records the reserved time intervals for each vertex, ensuring that the resulting trajectory $\tau_i$ avoids conflicts with higher-priority trajectories.

\noindent\textbf{Reservation Table:}
To accommodate planning in an edge-weighted graph, we adopt the idea from Safe Interval Path Planning (SIPP)~\cite{phillips2011sipp,ma2019lifelong} to reserve continuous-time intervals, instead of discrete time steps in the classic vertex-time A*~\cite{silver2005cooperative}. The reservation table is indexed by a vertex, with the corresponding entry recording the \textit{reserved intervals} during which the vertex is occupied.
Specifically, following the conflict definition (Definition~\ref{def:conflict}), given state $(v,t)$, the reservation table records the reserved interval $[t_p, t_s)$ for vertex $v$, where $(v_p,t_p)$ and $(v_s,t_s)$ are the preceding and succeeding states of state $(v,t)$, respectively. Notably, the reservation also records $[0, t_s)$ for the first state $(r_j,0)$ and $[t_p, +\infty)$ for the last state in each higher-priority trajectory $\tau_j$. The \textit{safe intervals} for each vertex $v$ are obtained as the complements of its reserved intervals with respect to interval $[0,+\infty)$.

\noindent\textbf{SIPP Node Expansion:} The low-level planner uses A* search to generate a search tree, where each node $n$ is uniquely identified by a pair $\langle v, [lb, ub)\rangle$, consisting of a vertex $v$ and one of its safe intervals $[lb, ub)$. The $g$-value $g(n) \in [lb, ub)$ of a node $n$ stores the earliest arrival time at which the robot can stay at vertex $v$ within the safe interval $[lb, ub)$ along any path. When expanding node $n=\langle v,[lb,ub)\rangle$, the search generates a child node $n'=\langle u, [lb',ub')\rangle$ with $g(n')$ for each vertex $u$ adjacent to $v$ and each safe interval $[lb', ub')$ of $u$, corresponding to the action of waiting at $v$ for the minimal duration $g(n')-w_{(v, u)}-g(n)\geq 0$, then departing from $v$ at time $g(n')-w_{(v,u)}$ and arriving at $u$ at the earliest available time $g(n')$. The search discards node $n'$ if no valid arrival time $g(n')$ is available such that $g(n')\geq lb'$ with the corresponding departure time $g(n')-w_{(v,u)}< ub$. If the search generates a duplicate copy of an already generated node, it keeps only the dominating one with the smallest $g$-value.

\noindent\textbf{Conflict-Free Trajectory Reconstruction:}
The trajectory $\tau_i$ for $\pi_i$ is generated by recursively backtracking the search tree that starts from $\langle\pi_i[|\pi_i|],\cdot\rangle$ and ends at $\langle\pi_i[1],\cdot\rangle$.
For each SIPP node $n=\langle v,\cdot\rangle$ along the search tree, we recover a state $\mathbf{x}=(v, g(n))$ and append it to $\tau_i$.


Given multiple ordered goals along a path $\pi$, we design three low-level planner approaches, \textit{chaining}, \textit{multi-label}, and \textit{adaptive}.

\noindent\textbf{Chaining Approach (Incomplete)}:
This approach iteratively computes a trajectory by constructing (trajectory) segments between consecutive vertices in $\pi$. Starting from vertex $\pi[1]$ with time $t_1=0$, each $j-th$ iteration (for $j=1,\ldots, |\pi|-1$) computes a segment from vertex $\pi[j]$ with the given arrival time $t_{j}$ to vertex $\pi[j+1]$ with a computed arrival time $t_{j+1}$. Each segment computation uses the aforementioned A* search, starting from the node with vertex $\pi[j]$ and a safe interval containing $t_j$. The computation succeeds if the search expands a goal node $n=\langle\pi[j+1],\cdot\rangle$ with arrival time $t_{j+1}=g(n)$; otherwise, it fails. The search uses the precomputed shortest-path distance from any vertex $v$ to $\pi[j+1]$ as the heuristic value for all nodes associated with $v$. The final trajectory, if all segments are successfully computed, is constructed by concatenating the segments. However, this chaining approach is incomplete~\cite{mouratidis2024fools} because it does not allow backtracking. For example, a successful trajectory may require searching from alternate safe intervals of $\pi[j]$ rather than the one containing $t_j$, which this approach does not account for.

\noindent\textbf{Multi-Label Approach (Complete and Optimal; Novel for Continuous Time)}:
This approach extends Multi-Label A*~\cite{grenouilleau2019multi,li2021lifelong,zhong2022optimal}, originally developed for unweighted graphs, to plan a time-minimal trajectory through ordered goals $\pi[j]$ ($j=1,\ldots, |\pi|$) in a single search. This approach adds a \textit{label} dimension to the state space and thus searches with nodes uniquely identified by a tuple $\langle v, [lb, ub), l\rangle$, consisting of a vertex $v$, one of the safe intervals $[lb, ub)$ of $v$, and a label $l$ indicating the current goal $\pi[l]$---resulting in a novel path-planning algorithm named Multi-Label SIPP (ML-SIPP). The search of ML-SIPP follows the same node expansion rule as standard SIPP, letting a child node inherit its label $l$ from its parent, except when its vertex is $\pi[l]$, in which case its label changes to $l+1$. The search terminates when it expands a (final-)goal node $\langle\pi[|\pi|], \cdot, |\pi|+1\rangle$. ML-SIPP uses the precomputed concatenated shortest-path distance heuristic:
\begin{align*}
h((v,\cdot,l))=\text{dist}(v,\pi[l])+\sum_{j=l}^{|\pi|-1}\text{dist}(\pi[j],\pi[j+1]).
\end{align*}
Since this heuristic is admissible---never overestimating the cost of reaching the (final-)goal node---ML-SIPP is a complete and optimal low-level planner.

\noindent\textbf{Adaptive Approach (Complete; Novel):}
The adaptive approach, shown in Alg.~\ref{alg:ada}, seeks to enhance practicability by combining the efficiency of the chaining approach and the completeness of ML-SIPP. It sequentially applies ML-SIPP to subsequences of goals in path $\pi$, each time starting from a subsequence of a single goal [Line \ref{alg:ada:update_S}] and backtracking to include the previous goal whenever ML-SIPP fails to find a trajectory for the current subsequence [Lines~\ref{alg:ada:limited_window}-\ref{alg:ada:add2seq}].
The backtracking depth is limited by a small integer $b_\text{max}\in\mathbb{Z}^+$ (we set it as $5$ in Sec.~\ref{sec:res}).
Upon successfully planning for a single goal, the approach resets current backtracking depth $b$ [Lines~\ref{alg:ada:unsolvable}-\ref{alg:ada:update_S}].
Our adaptive approach is complete since it calls ML-SIPP for the entire sequence of goals in $\pi$ (i.e., the multi-label approach) in the worst case when ML-SIPP fails with the window size $b_\text{max}$ [Lines~\ref{alg:ada:fallback_1}-\ref{alg:ada:fallback_2}].
In practice, we postpone the ML-SIPP call on the entire goal sequence and mark the PBS node $N$ for the multi-label approach.
If the high-level PBS tree reports failure once its search stack is empty, we pick one postponed node, unmark it, and call the multi-label approach for the node with children nodes branched.
This strategy greatly improves the overall efficiency of the deconflicting procedure, as each call of the multi-label approach is highly time-consuming, especially when dealing with large-scale instances.

\begin{algorithm}[t]
\DontPrintSemicolon
\linespread{0.95}\selectfont
\caption{Adaptive-Low-Level-Planner}\label{alg:ada}
\SetKwInput{KwInput}{Input}
\SetKwInput{KwParam}{Param}
\KwInput{path $\pi=(v_1,\ldots,v_{|\pi|})$ for robot $i$,\quad\quad reservation table $rt$ w.r.t. $\pmb{\prec}_N$ of PBS node $N$}
\KwParam{max backtracking size $b_\text{max}\in\mathbb{Z}^+$}
remove every $v\in R\setminus\{r_i\}$ from $\pi$\;
$\mathbf{x}_\text{last}\gets\left(v_1,0\right)$,\, $j=2$,\, $b\gets 1$\;
$\tau\gets\left(\mathbf{x}_\text{last},\right)$,\, $S_\text{goal}\gets (v_j,\,)$\;
\While{$j \leq |\pi|$}{
find a trajectory from $\mathbf{x}_\text{last}$ to $S_\text{goal}$ via ML-SIPP\;
\uIf{a valid trajectory $\Tilde{\tau}$ is found}{\label{alg:ada:unsolvable}
    append $\Tilde{\tau}$ to the end of $\tau$\;\label{alg:ada:append2P}
    $j\gets j+1$,\, $b\gets 1$\;\label{alg:ada:update_b}
    $S_\text{goal}\gets (v_j,\,)$\;\label{alg:ada:update_S}
}
\uElseIf{$b\leq b_\text{max}$ or $\tau$ is not empty}{\label{alg:ada:limited_window}
    pop the last state from $\tau$ until $v_{j-b}$\;\label{alg:ada:pop_state}
    $S_\text{goal}\gets (v_{j-b},\,)+S_\text{goal}$,\, $b\gets b+1$\;\label{alg:ada:add2seq}
}
\Else{
mark $N$ as postponed for multi-label approach\;\label{alg:ada:fallback_1}
\Return$\emptyset$\;\label{alg:ada:fallback_2}
}

$\mathbf{x}_\text{last}\gets$ the last state in $\tau$\;
}
\Return $\tau$
\end{algorithm}

\subsection{Integrating Turning Costs}\label{sec:turn_cost}

So far, the trajectory computation assumes the robots are holonomic, that is, they can move omnidirectional without the need to change their headings between consecutive move actions.
We now show that one can remove this assumption for non-holonomic robots by adding an \textit{orientation} $\psi\in \{\text{E}, \text{N}, \text{W}, \text{S}\}$ to the state space, restricted to East (E), North (N), West (W), and South (S) in the four-neighbor grid graph $G$.
Correspondingly, we allow a robot to \textit{turn} at vertex $v$ from $\psi$ to $\psi'\neq\psi$ at a time cost of $t'-t=C\cdot\frac{|\psi'-\psi|}{90^\circ}>0$, where $C$ is a constant cost for a $90^\circ$ turns. Effectively, each robot performs a \textit{turn-wait-move} action.
With this modification, one can deconflict an MCPP solution for non-holonomic robots.

Additionally, the above modification requires the parallel rewiring optimization for ESTC to consider the turn costs when calculating the path cost difference (Sec.~\ref{subsec:local_opt}). Other components of LS-MCPP (Alg.~\ref{alg:ls-mcpp}), including ESTC (Alg.~\ref{alg:estc}) and its turn reduction optimization, all remain unchanged.

\begin{figure*}[t]
\centering
\includegraphics[width=\linewidth]{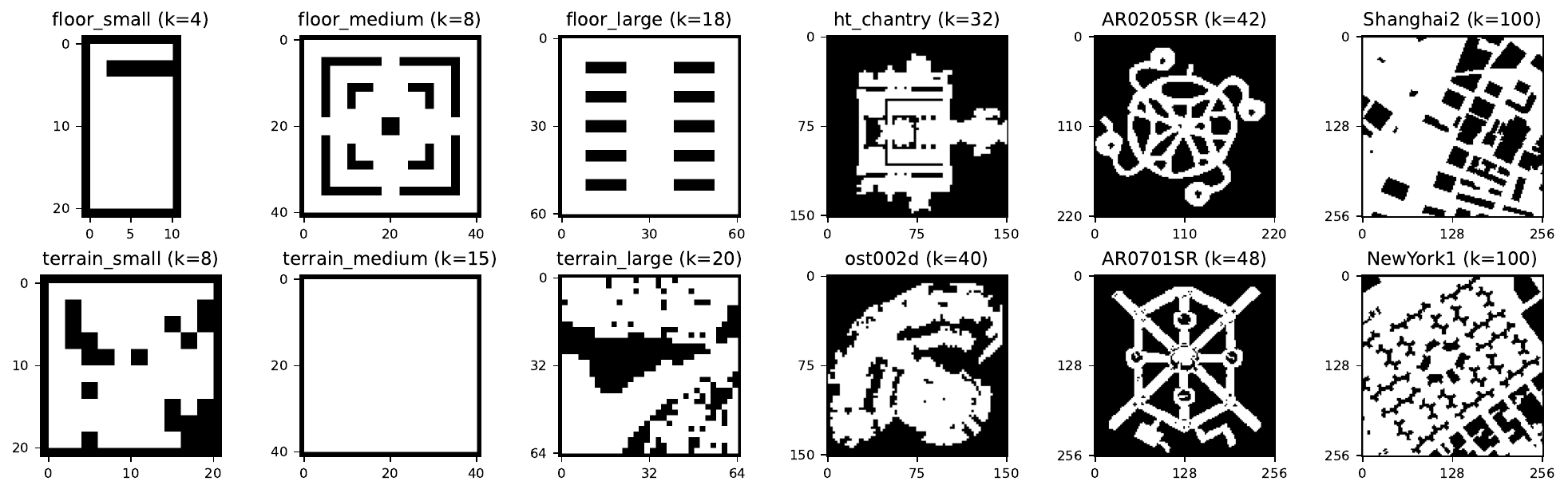}
\caption{MCPP instances with $k$ robots for the simulation experiments. Top row: unweighted grid graphs. Bottom row: weighted grid graphs.}
\label{fig:instances}
\end{figure*}

\section{Experiments and Results}~\label{sec:res}
This section describes our numerical results of simulations on an Ubuntu 22.04 PC with a 2.50GHz \textit{Intel}\textsuperscript{\textregistered} \textit{Core}\textsuperscript{\textregistered} i5-13490F CPU and 32GB RAM.
Our proposed algorithms and the baseline methods are publicly available at Github\footnote{\url{https://github.com/reso1/LS-MCPP}}, with all the benchmark instances and reported numerical results provided as well.
In Sec.~\ref{sec:exp}, we deployed the proposed algorithms on two \textit{Pepper}\footnote{\url{https://www.aldebaran.com/en/pepper}} robots in an indoor environment to validate our planning pipeline for MCPP.

\subsection{Experiment Setup}\label{subsec:benchmark}

\noindent\textbf{Instances:} 
We generate a set of instances to benchmark different methods utilizing the four-way connected 2d undirected grid graphs from previous grid-based MCPP works~\cite{tang2023mixed, tang2024large} as well as the 2D pathfinding benchmark~\cite{sturtevant2012benchmarks}.
As shown in Fig.~\ref{fig:instances}, the graphs in the top row have unweighted edges with a uniform weight of $1$, whereas the graphs in the bottom row have weighted edges whose weights are randomly generated from $1$ to $3$.
Originally, each graph can be contracted to a hypergraph where all its hypervertices are complete.
The original graphs are then used to generate 12 groups of 12 mutated MCPP instances (i.e., mutations) of $k$ robots, resulting from a removal ratio $\rho=0,1,...,11$ and a seed $s=0,1,...,11$.
Each mutation $(G,I,R)$ with $\rho$ and $s$ is constructed by (1) randomly sampling (with seed $s$) a set of $k$ vertices from $G$ as $R$ and (2) randomly removing (with seed $s$) vertices from $G$ by $x=(\frac{100}{12}\cdot\rho)\%$, meanwhile ensuring the root vertex of any robot and all its neighbors are not removed.
The root vertices in $R$ are sampled to ensure they are not cut vertices for constructing well-formed instances~\cite{vcap2015complete,ma2019searching}, that is, a robot can wait indefinitely at its root vertex without blocking any other robot.
Note that the graph edge weights of each mutation are also randomly generated using its corresponding seed $s$.
For all mutations, we set $C=0.5$ as the $90^\circ$ turning cost as described in Sec.~\ref{sec:turn_cost}.

\begin{figure}[t]
\centering
\includegraphics[width=\linewidth]{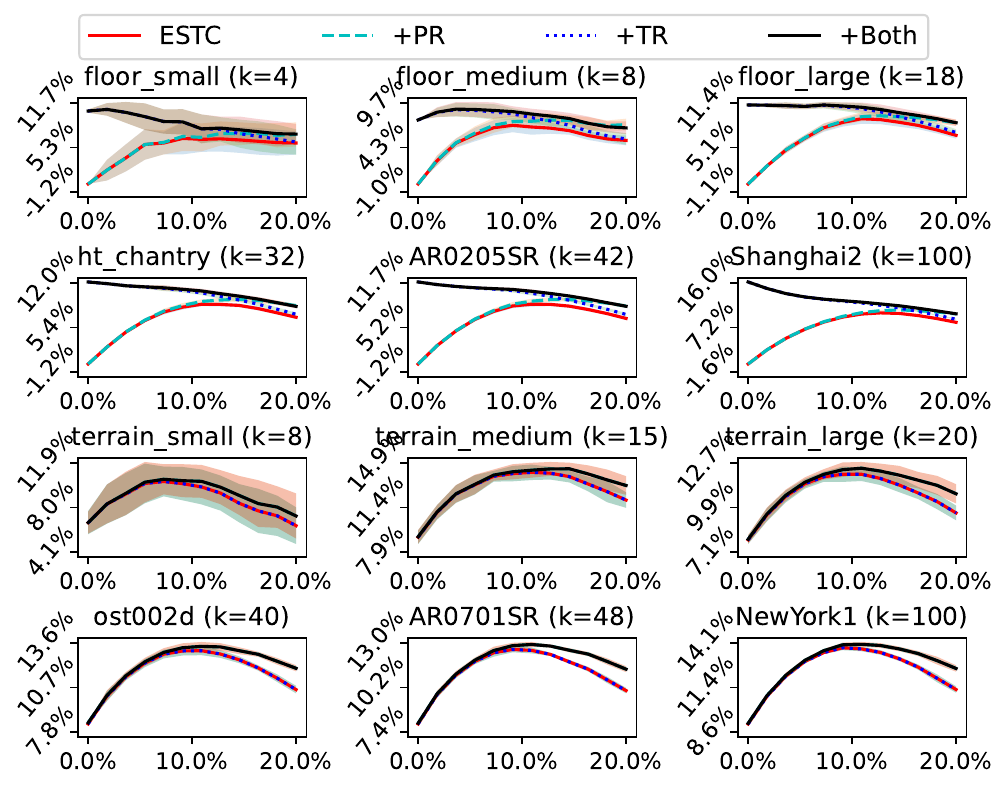}
\caption{Ablation study for single-robot CPP. Each subplot displays the removal ratio $\rho$ (x-axis) versus cost reductions compared to ESTC-UW (y-axis). The cost statistics (mean and variance) are computed across all 12 seeds.}
\label{fig:CPP_ablations}
\end{figure}

\noindent\textbf{Baselines}:
We describe all the methods being compared in our experiments.
(1) \textit{Single-Robot CPP}: To verify the effectiveness of our ESTC with the proposed hyperedge weight definition, we use a baseline method, namely ESTC-UW, that computes the circumnavigating coverage path on a spanning tree assuming the edges are unweighted.
We also performed an ablation study on the two local optimizations for ESTC: parallel rewiring (PR) in Sec.~\ref{subsec:estc} and turn reduction (TR) in Sec.~\ref{subsec:local_opt}.
(2) \textit{MCPP}: As mentioned in Sec.~\ref{subsec:apply_estc}, the proposed ESTC algorithm can be integrated into various grid-based MCPP methods, making them immediately applicable to MCPP.
To reiterate, VOR, MFC, MIP, and MSTC$^*$ are the comparing baseline methods with our LS-MCPP.
In addition, we also compare LS-MCPP with and without the vertex-wise operators that are new to our previous work~\cite{tang2024large}~\footnote{In the ablation study of our previous work~\cite{tang2024large}, we have shown the effectiveness of grow, deduplicate, and exchange operators as well as the use of Forced Deduplication (Alg.~\ref{alg:fd}) and heuristic sampling for LS-MCPP.}, namely LS(+VO) and LS(-VO) respectively.
Note that all the methods compared use ESTC with TR and PR applied by default.
(3) \textit{Deconflicted MCPP}: Given LS(+VO) solutions to produce the initial time-embedded trajectories, we compare the deconflicting performance using the chaining approach (CHA), the multi-label approach (MLA), and the adaptive approach (ADA) as the low-level planner for PBS.

\noindent\textbf{Parameters}:
Given an MCPP instance $(G,I,R)$, we describe the settings of all the parameters used in the simulation experiments.
For our LS-MCPP (Alg.~\ref{alg:ls-mcpp}), we set the maximum number of iterations $M=1.0e^3\times\sqrt{|V|/k}$ for a good tradeoff between planning efficiency and solution quality, the forced deduplication step size $S=\lfloor M/20 \rfloor$ to arrange at least 20 \textit{Forced-Deduplication} calls, the temperature decay factor $\alpha=\exp{\frac{\log 0.2}{M}}$ to decrease the temperature $t$ from $1$ to $0.2$, and the pool weight decay factor $\gamma=1.0e^{-2}$.
For MIP, we set a runtime limit of one hour for the mixed-integer program solving.
For the Deconflicted MCPP algorithms in Sec.~\ref{subsec:low-level-planner}, we set a total runtime limit of one hour for the overall procedure.
Specifically for the adaptive approach (Alg.~\ref{alg:ada}), we find that setting the max backtracking size $b_\text{max}=5$ is sufficient to resolve all local conflicts efficiently.

\subsection{Ablation Study}\label{subsec:ablation}


\noindent\textbf{Single-Robot CPP:}
In Fig.~\ref{fig:CPP_ablations}, we compare all four ESTC variants with the baseline ESTC-UW. 
We observe a more critical cost reduction on weighted graphs (upper 6 subplots) than on unweighted graphs (lower 6 subplots).
For unweighted graphs, TR results in significant performance improvements since uniformly weighted edges can be aligned in one direction to minimize turning costs. 
For weighted graphs, TR is much less effective as there are few tie-breaking cases for reducing the turning costs, while PR achieves higher cost reductions as the vertex removal percentage increases.
In general, ESTC shows its effectiveness for single-robot CPP, and both PR and TR contribute to substantial path cost reductions.

\begin{figure}[t]
\centering
\includegraphics[width=\linewidth]{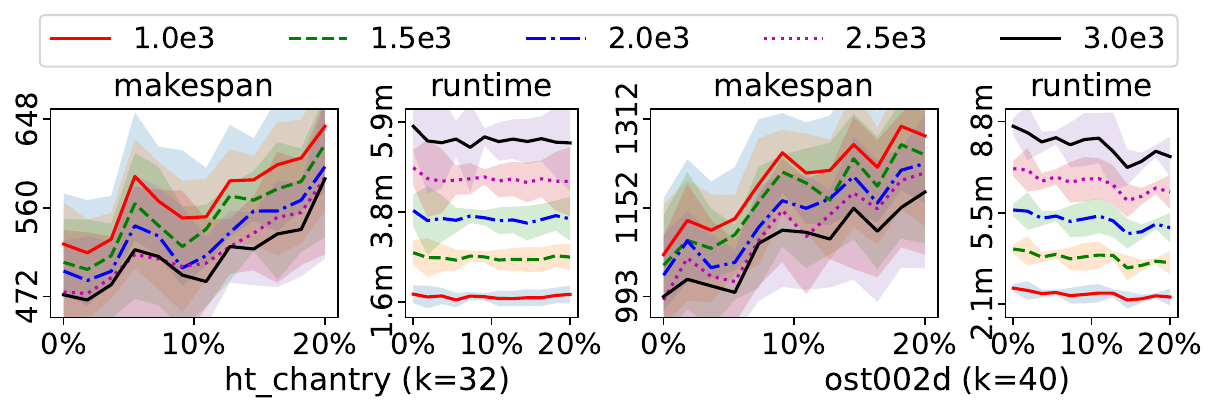}
\caption{Performance of LS(+VO) solutions with varying iteration limits. The maximum number of iterations is scaled as $M=b\times\sqrt{|V|/k}$, where different colored lines represent different base scalers $b$.}
\label{fig:diff_iters}
\end{figure}

\begin{figure*}[t]
\centering
\includegraphics[width=0.95\linewidth]{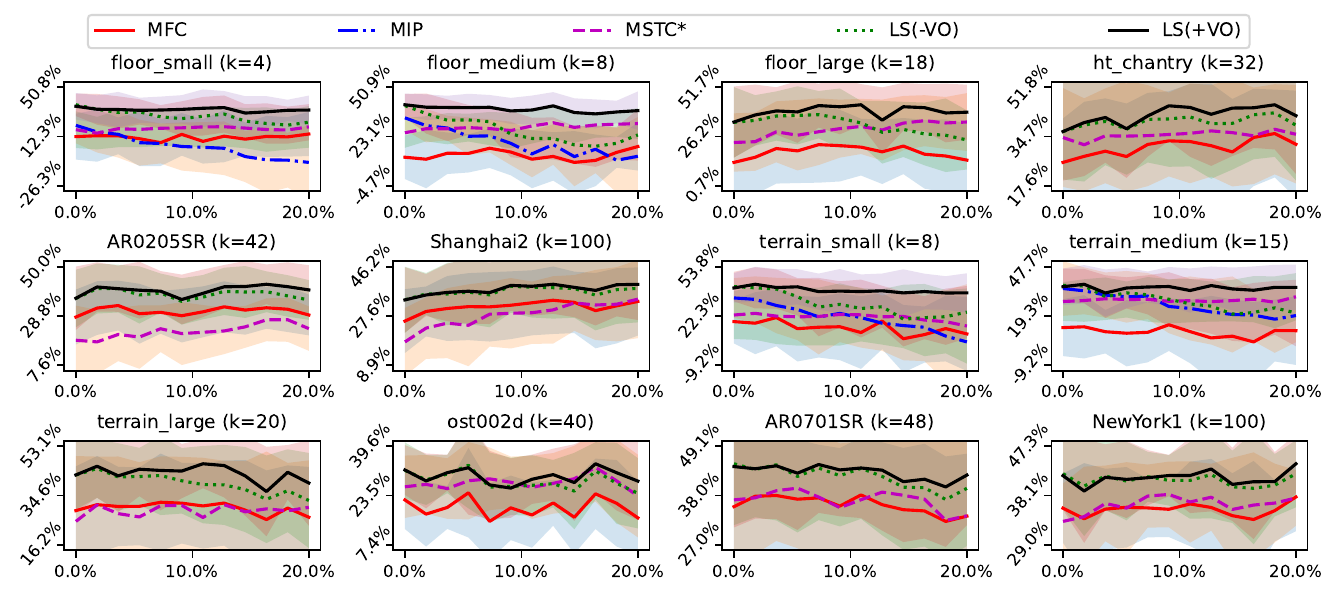}
\caption{Ablation study for MCPP. Each subplot displays the removal ratio $\rho$ (x-axis) versus makespan reductions compared to VOR (y-axis). The makespan statistics (mean and variance) are computed across all 12 seeds.}
\label{fig:MCPP_ablations}
\end{figure*}

\begin{figure*}
\centering
\includegraphics[width=0.95\linewidth]{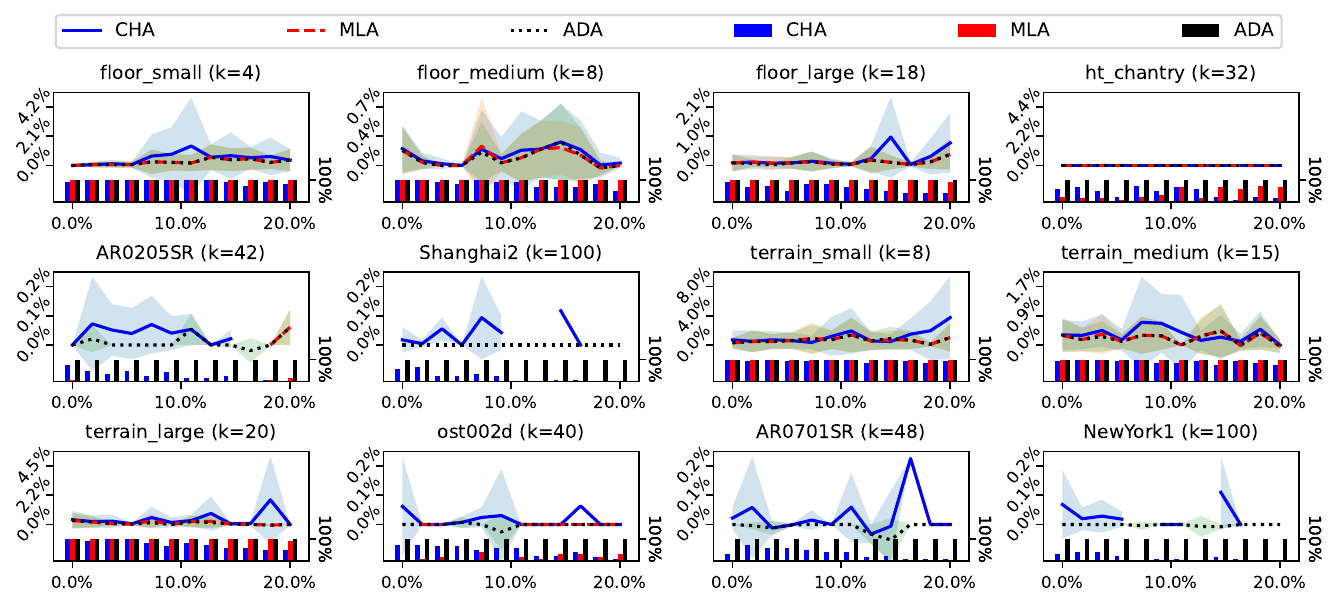}
\caption{Ablation study for Deconflicted MCPP. Each subplot displays the removal ratio $\rho$ (x-axis) versus makespan increments relative to the input LS(+VO) MCPP solution (left y-axis) and success rates (right y-axis). The makespan statistics (mean and variance) are computed across seeds where CHA, MLA, and ADA all succeed; when certain approach(es) fail in all seeds, statistics are computed across seeds where the remaining approach(es) succeed.
}
\label{fig:DecMCPP_ablations}
\end{figure*}

\noindent\textbf{MCPP:}
We compare the solution quality of MCPP between different methods in Fig.~\ref{fig:MCPP_ablations}.
Generally speaking, the mutated instances become more complex as the vertex removal ratio increases.
Our LS-MCPP consistently outperforms other methods across all instances, while the performance of the other methods varies, with each showing advantages in different instances.
Note that due to the huge memory and runtime requirements of MIP, we only test it on relatively smaller instances (i.e., names with a suffix of \textit{small} or \textit{medium}), and its performance is highly restricted by the runtime limit of 1 hour.
Specifically for our LS-MCPP, as the vertex removal ratio increases, LS(+VO) performs better with the additional vertex operators compared to LS(-VO).
However, when the instances become more complex with a large map and more robots, the solution quality of LS(+VO) gets closer to LS(-VO) due to an inadequate number of local search iterations limited by $M$.
To demonstrate that we can further improve the solution quality by saturating LS(+VO) with more iterations, we compare the LS(+VO) solutions with increasing maximum numbers of iterations, as shown in Fig.~\ref{fig:diff_iters}.
The solution quality steadily improves as the maximum iteration limit $M$ increases with a linear increase in runtime.
In practice, $M$ can be flexibly adjusted according to the available runtime budget, allowing a trade-off between better solution quality and faster planning.

\noindent\textbf{Deconflicted MCPP:}
We explain the Deconflicted MCPP ablation study results in Fig.~\ref{fig:DecMCPP_ablations}.
CHA demonstrates the highest makespan increments and the lowest success rate due to its inability to backtrack to previous goal states, resulting in incompleteness. 
This limitation makes CHA susceptible to ``dead-ends'' when no valid succeeding states are available. 
Consequently, it reports a higher number of failure nodes during PBS search, which either prevents PBS from finding a valid priority ordering or significantly increases the number of PBS nodes explored (sometimes exceeding the runtime limit in large-scale instances).
MLA produces high-quality trajectory sets, with no makespan increases for most of the ``easier'' instances it can handle within the runtime limit. 
However, for more complex instances involving larger maps and more robots, the added label dimension across the entire coverage path expands the search space considerably, which significantly increases planning time for each PBS node and results in a lower success rate due to timeouts.
Our ADA succeeds in all instances, achieving high solution quality with no makespan increases in most cases. 
This is due to its ability to restrict the label-augmented search space within a local window, rather than across the entire coverage path as MLA, and only when backtracking to previous states is necessary.
In Fig.~\ref{fig:diff_sols}, we also compare the deconflicting results of different MCPP solutions for \textit{floor\_large} and \textit{terrain\_large} using ADA as the low-level planner for PBS.
The results demonstrate the better solution quality of the proposed LS-MCPP when used as the input MCPP solution for deconfliction.
In contrast, as MFC and MSTC$^*$ have larger makespans due to the higher overlapping covered regions across the map, the deconflicting procedures on them are more time-consuming.
Notably, there is one failure case of deconflicting the MFC solution where the one-hour runtime limit is exceeded.
For VOR, as there are no overlapping regions in its solution, PBS immediately returns the initial time-embedded trajectories without any conflict.

\begin{figure}[t]
\centering
\includegraphics[width=\linewidth]{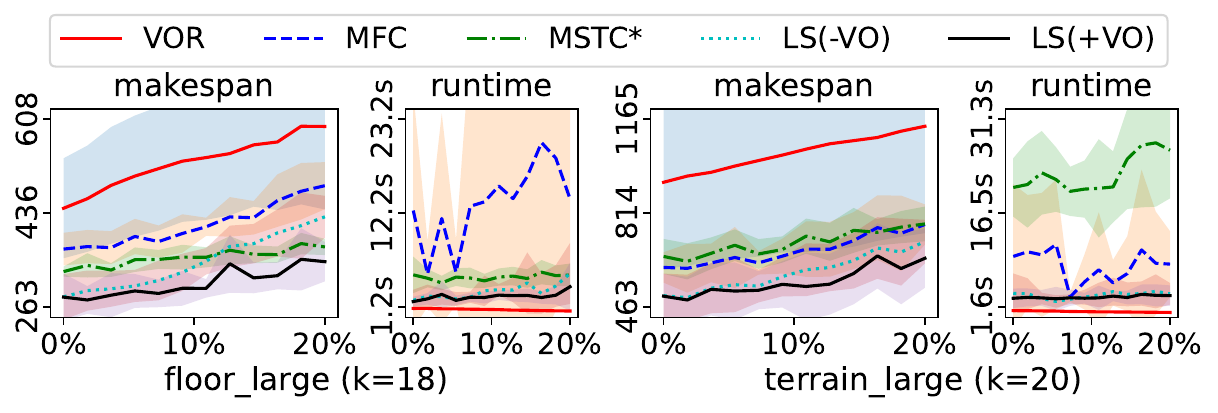}
\caption{Deconflicting different MCPP solutions using ADA as the low-level planner for PBS. The left two subplots show the removal ratio $\rho$ (x-axis) versus makespan (left y-axis) and success rates (right y-axis). The makespan statistics (mean and variance) are computed across seeds where ADA succeeds for all input MCPP solutions.}
\label{fig:diff_sols}
\end{figure}

\subsection{Planning Runtime Comparison}
We report the planning runtime of all the MCPP methods and the Deconflicted MCPP methods across all the 12 groups of instances in Fig.~\ref{fig:instances}.
We show the planning runtime comparison of methods for MCPP in Tab.~\ref{tab:relaxed_mcpp_running_time}, where our LS-MCPP requires more time to explore the solution neighborhood.
In particular, adding vertex-wise operators only increases the runtime of LS-MCPP by a little if we compare LS(-VO) and LS(+VO).
Nevertheless, it produces high-quality MCPP solutions within around 13 minutes even for large-scale instances such as \textit{NewYork1} and \textit{Shanghai2}, which is acceptable for offline planning in real-world robotics applications.

In Tab.~\ref{tab:dec_mcpp_running_time},  we compare the planning runtime of the three low-level planners with PBS as the high-level planner, using the LS(+VO) solution as input.
Each approach-instance data cell displays the runtime for instances where all three approaches succeed at the top, and the runtime for instances where each approach succeeds is shown at the bottom. 
Note that for the top item, when one or more approaches fail for all 12 instances, we also include the runtime for instances where the other approaches all succeed.
We can see that for the relatively easier instances where all three approaches succeed at least one mutation, CHA is slightly faster than ADA, whereas MLA is the slowest.
For more complex instances (e.g., \textit{NewYork1}) where CHA only succeeds at a few mutations (see Fig.~\ref{fig:DecMCPP_ablations}), ADA can still produce a set of valid trajectories for all of them within 10 minutes, and sometimes even faster than CHA who needs to explore more PBS nodes to find a successful set of trajectories.

\begin{table}[t]
\caption{Runtime Comparison of MCPP Methods.}\label{tab:relaxed_mcpp_running_time}
\centering
\setlength\tabcolsep{0.3pt}
\renewcommand{\arraystretch}{1.0}
\begin{tabular}{|c|c|c|c|c|c|c|}
\hline
 & \textit{\textbf{\begin{tabular}[c]{@{}c@{}}floor\\ small\end{tabular}}}   & \textit{\textbf{\begin{tabular}[c]{@{}c@{}}floor\\ medium\end{tabular}}}   & \textit{\textbf{\begin{tabular}[c]{@{}c@{}}floor\\ large\end{tabular}}}   & \textit{\textbf{\begin{tabular}[c]{@{}c@{}}ht\_ch\\ antry\end{tabular}}} & \textit{\textbf{\begin{tabular}[c]{@{}c@{}}AR02\\ 05SR\end{tabular}}} & \textit{\textbf{\begin{tabular}[c]{@{}c@{}}Shang\\ hai2\end{tabular}}} \\ \hline
\textbf{VOR} & 0.01$\pm$0.01s & 0.03$\pm$0.01s & 0.1$\pm$0.01s & 0.2$\pm$0.03s & 0.4$\pm$0.04s & 1.5$\pm$0.1s \\ \hline
\textbf{MFC} & 0.02$\pm$0.01s & 0.21$\pm$0.02s & 1.0$\pm$0.08s & 6.8$\pm$0.11s & 20$\pm$1.35s & 4.7$\pm$0.3m \\ \hline
\textbf{MIP} & 1h & 1h & / & / & / & / \\ \hline
\textbf{MSTC*} & 0.01$\pm$0.01s & 0.12$\pm$0.05s & 0.6$\pm$0.24s & 3.5$\pm$1.63s & 9.6$\pm$4.4s & 1.7$\pm$0.7m \\ \hline
\textbf{LS(-VO)} & 2.38$\pm$2.50s & 0.57$\pm$0.2m & 0.9$\pm$0.1m & 1.8$\pm$0.2m & 3.0$\pm$0.3m & 9.9$\pm$1.5m \\ \hline
\textbf{LS(+VO)} & 2.87$\pm$1.89s & 0.69$\pm$0.1m & 0.9$\pm$0.1m & 2.1$\pm$0.2m & 3.4$\pm$0.3m & 13$\pm$0.8m \\ \hline
 & \textit{\textbf{\begin{tabular}[c]{@{}c@{}}terrain\\ small\end{tabular}}} & \textit{\textbf{\begin{tabular}[c]{@{}c@{}}terrain\\ medium\end{tabular}}} & \textit{\textbf{\begin{tabular}[c]{@{}c@{}}terrain\\ large\end{tabular}}} & \textit{\textbf{ost002d}} & \textit{\textbf{\begin{tabular}[c]{@{}c@{}}AR07\\ 01SR\end{tabular}}} & \textit{\textbf{\begin{tabular}[c]{@{}c@{}}New\\ York1\end{tabular}}}  \\ \hline
\textbf{VOR} & 0.01$\pm$0.01s & 0.04$\pm$0.01s & 0.1$\pm$0.01s & 0.4$\pm$0.04s & 0.6$\pm$0.1s & 1.6$\pm$0.1s \\ \hline
\textbf{MFC} & 0.04$\pm$0.01s & 0.40$\pm$0.04s & 1.1$\pm$0.22s & 15$\pm$1.46s  & 0.7$\pm$0.1m & 5.4$\pm$0.2m \\ \hline
\textbf{MIP} & 1h & 1h & / & / & / & / \\ \hline
\textbf{MSTC*} & 0.04$\pm$0.02s & 0.50$\pm$0.21s & 1.4$\pm$0.55s & 15$\pm$6.71s & 0.5$\pm$0.3m & 2.6$\pm$0.8m \\ \hline
\textbf{LS(-VO)} & 5.88$\pm$3.41s & 0.46$\pm$0.1m & 0.6$\pm$0.2m & 2.6$\pm$0.3m & 4.9$\pm$0.5m & 9.1$\pm$0.5m \\ \hline
\textbf{LS(+VO)} & 6.65$\pm$2.86s & 0.51$\pm$0.1m & 0.7$\pm$0.1m & 3.0$\pm$0.3m & 5.7$\pm$0.4m & 13$\pm$0.5m \\ \hline
\end{tabular}
\end{table}

\begin{table}[t]
\caption{Runtime Comparison of Deconflicted MCPP Methods.}\label{tab:dec_mcpp_running_time}
\centering
\setlength\tabcolsep{1pt}
\renewcommand{\arraystretch}{1.0}
\begin{tabular}{|c|c|c|c|c|c|c|}
\hline
 & \textit{\textbf{\begin{tabular}[c]{@{}c@{}}floor\\ small\end{tabular}}} & \textit{\textbf{\begin{tabular}[c]{@{}c@{}}floor\\ medium\end{tabular}}} & \textit{\textbf{\begin{tabular}[c]{@{}c@{}}floor\\ large\end{tabular}}} & \textit{\textbf{\begin{tabular}[c]{@{}c@{}}ht\_ch\\ antry\end{tabular}}} & \textit{\textbf{\begin{tabular}[c]{@{}c@{}}AR02\\ 05SR\end{tabular}}} & \textit{\textbf{\begin{tabular}[c]{@{}c@{}}Shang\\ hai2\end{tabular}}} \\ \hline
\multirow{2}{*}{\textbf{CHA}} & 0.10$\pm$0.1s & 0.7$\pm$0.2s & 6.9$\pm$2.7s & 0.5$\pm$0.7m & 1.6$\pm$1.5m & 9.0$\pm$8.4m \\ \cline{2-7} 
                              & 0.10$\pm$0.1s & 0.7$\pm$0.2s & 6.9$\pm$2.7s & 0.9$\pm$2.7m & 1.5$\pm$1.4m & 9.0$\pm$8.4m \\ \hline
\multirow{2}{*}{\textbf{MLA}} & 0.16$\pm$0.1s & 4.8$\pm$3.9s & 3.7$\pm$4.4m & 25$\pm$23m   & 3.2$\pm$0.0m & /            \\ \cline{2-7} 
                              & 0.17$\pm$0.2s & 5.2$\pm$4.5s & 4.0$\pm$4.8m & 34$\pm$21m   & 31$\pm$28m   & /            \\ \hline
\multirow{2}{*}{\textbf{ADA}} & 0.10$\pm$0.1s & 0.7$\pm$0.2s & 7.3$\pm$3.2s & 0.4$\pm$0.5m & 1.4$\pm$1.4m & 9.1$\pm$6.3m \\ \cline{2-7} 
                              & 0.10$\pm$0.1s & 0.8$\pm$0.2s & 8.1$\pm$4.1s & 0.6$\pm$0.6m & 1.4$\pm$1.2m & 7.0$\pm$5.2m \\ \hline
 & \textit{\textbf{\begin{tabular}[c]{@{}c@{}}terrain\\ small\end{tabular}}} & \textit{\textbf{\begin{tabular}[c]{@{}c@{}}terrain\\ medium\end{tabular}}} & \textit{\textbf{\begin{tabular}[c]{@{}c@{}}terrain\\ large\end{tabular}}} & \textit{\textbf{ost002d}} & \textit{\textbf{\begin{tabular}[c]{@{}c@{}}AR07\\ 01SR\end{tabular}}} & \textit{\textbf{\begin{tabular}[c]{@{}c@{}}New\\ York1\end{tabular}}} \\ \hline
\multirow{2}{*}{\textbf{CHA}} & 0.18$\pm$0.1s & 1.3$\pm$0.6s & 9.3$\pm$3.7s & 1.0$\pm$0.6m & 2.4$\pm$1.6m & 5.8$\pm$7.0m \\ \cline{2-7} 
                              & 0.18$\pm$0.1s & 1.3$\pm$0.6s & 9.4$\pm$3.8s & 1.3$\pm$1.1m & 2.4$\pm$1.6m & 5.8$\pm$7.0m \\ \hline
\multirow{2}{*}{\textbf{MLA}} & 0.27$\pm$0.2s & 5.9$\pm$13s  & 5.4$\pm$12m  & 24$\pm$24m   & /            & /            \\ \cline{2-7} 
                              & 0.28$\pm$0.3s & 5.8$\pm$12s  & 5.2$\pm$12m  & 38$\pm$21m   & /            & /            \\ \hline
\multirow{2}{*}{\textbf{ADA}} & 0.20$\pm$0.1s & 1.3$\pm$0.5s & 9.9$\pm$4.7s & 0.9$\pm$0.5m & 2.1$\pm$1.0m & 4.4$\pm$3.0m \\ \cline{2-7} 
                              & 0.20$\pm$0.2s & 1.3$\pm$0.5s & 10$\pm$4.6s  & 1.2$\pm$0.6m & 2.1$\pm$0.8m & 4.2$\pm$2.4m \\ \hline
\end{tabular}
\end{table}

\subsection{Physical Robot Experiments}\label{sec:exp}

We constructed an MCPP instance of an indoor environment with a \textit{Vicon}\footnote{\url{https://www.vicon.com}} motion capture system for localization, as shown in Fig.~\ref{fig:real_exp}.
Our pipeline computes the conflict-free set of trajectories for the robots by first finding an MCPP solution via LS-MCPP and then deconflicting the MCPP solution via PBS and ADA.
After offline planning, the trajectories are sent to the robots to execute the coverage task via a centralized scheduler implemented in \textit{ROS}\footnote{\url{https://www.ros.org}}.
Fig.~\ref{fig:real_exp} demonstrates the execution procedure of the resulting trajectories using our planning pipeline.
The detailed setup is described as follows.

\noindent\textbf{Multi-Robot System:} We employ two Pepper robots, named \textit{Cayenne} and \textit{Cumin}, each with a square bounding box of 0.55 meters. 
The 2D grid graph represents the indoor environment as a $10\times 6$ map with each grid unit sized at 0.65 meters to accommodate their bounding boxes.
The initial root vertices of \textit{Cayenne} and \textit{Cumin} are the bottom-left and top-right cells, respectively.
Both Peppers can have omnidirectional movement driven by velocity commands $u=(v_x,v_y,\omega)$.
The linear velocities $v_x$ and $v_y$ are capped at 0.15 meters per second, while the angular velocity $\omega$ is restricted to 0.5 radians per second.
To determine the non-uniform traversal costs (i.e., the edge weights) for Peppers, we recorded the maximum and minimum time costs to traverse a single grid unit by performing random move actions across the graph. 
We then sampled weights randomly within this range and assigned them to the edges.
Although Peppers can move omnidirectionally, we assume they always face the moving directions by integrating turn actions as described in Sec.~\ref{sec:turn_cost}.
We set $C$ as the maximum action time of 90$^\circ$ rotations for Peppers.
In the final constructed graph, the time costs for move actions range from 7.2 to 10.8 seconds, while the time cost for a 90$^\circ$ turn is 6.0 seconds.

\begin{figure*}[t]
\centering
\includegraphics[width=\linewidth]{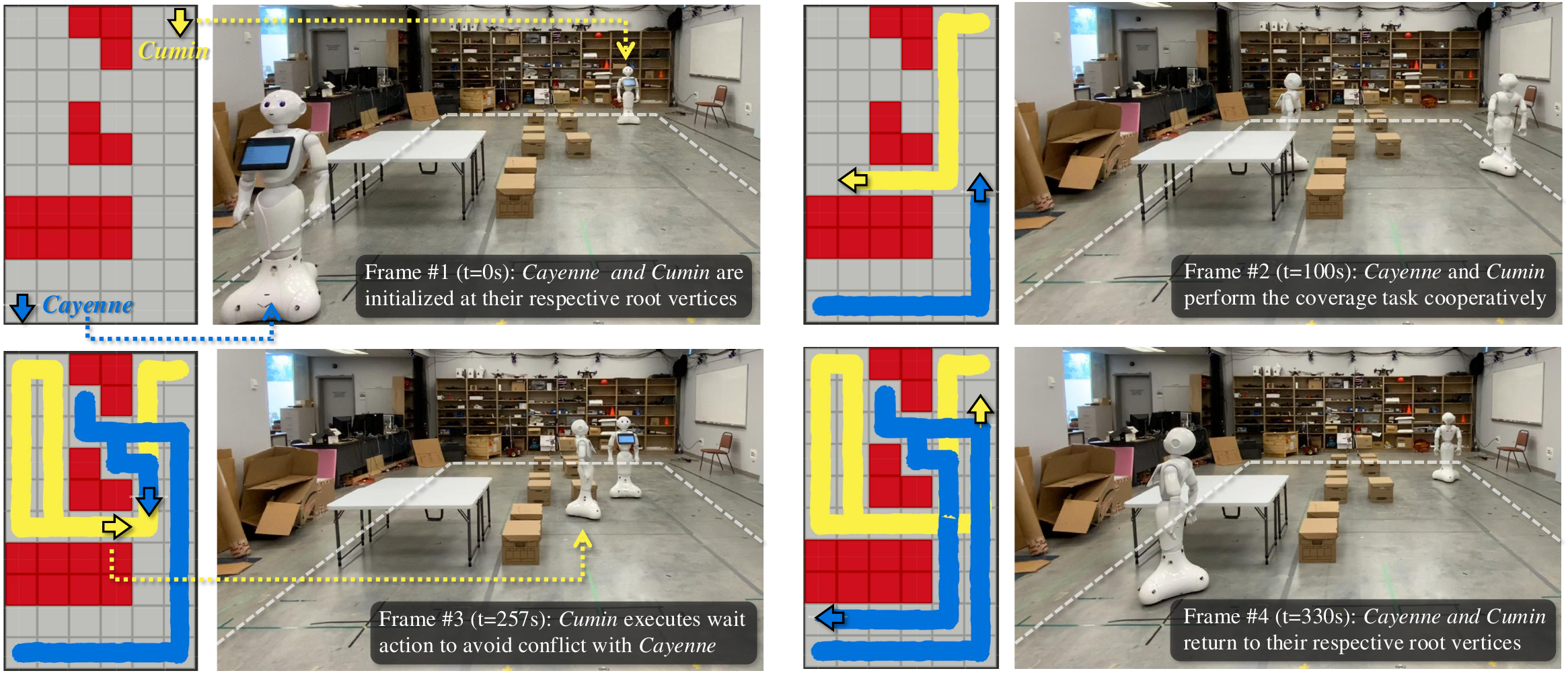}
\caption{Demonstration of \textit{Cayenne} and \textit{Cumin} executing the coverage task cooperatively and without conflicts. Four frames from the execution are shown, with robot trajectories and the grid graph displayed on the left. Static obstacles and the trajectories of \textit{Cayenne} and \textit{Cumin} are highlighted in red, blue, and yellow, respectively. The entire planning pipeline was completed in 0.1 seconds, and the execution of the coverage task took approximately 6 minutes.}
\label{fig:real_exp}
\end{figure*}

\noindent\textbf{Centralized Scheduler:}
Given the set of conflict-free trajectories resulting from our planning pipeline, we generate an action sequence consisting of move, turn, and wait actions between consecutive states for each Pepper.
A centralized scheduler is employed to sequentially send actions to the corresponding Pepper, maintaining a positional tolerance of 0.01 meters and a rotational tolerance of 0.02 radians.
The velocity commands for each action are generated using a proportional controller operating at 20 Hz to align with the nominal action time cost.
In practice, the move and turn action time costs are not precise for the Peppers, as we assume instantaneous velocity changes and do not account for their real-time dynamics or potential perturbations. 
This often leads to an overestimation of the actual action time costs, as the time for acceleration and deceleration is ignored, though underestimations can also occur occasionally.
Therefore, we design the centralized scheduler to synchronize the discrepancies between the nominal time costs $t$ and the actual time costs $t'$ for the actions.
Upon completing each action, each Pepper sends the time difference $\Delta t=t-t'$  back to the scheduler.
If $\Delta t <0$, the scheduler signals the Pepper to wait for a period of $-\Delta t$ to match the nominal time cost; otherwise, the scheduler instructs the other Pepper to add $\Delta t$ to its waiting period after completing its current action.

\begin{remark}
Our centralized scheduler currently uses a basic approach to synchronize time differences between nominal and actual action time costs, which may increase the task completion time. 
More advanced techniques from MAPF execution research, such as constructing an action dependency graph~\cite{honig2019persistent}, could address this issue. 
However, exploring these advanced methods is beyond the scope of this paper.
\end{remark}

\section{Conclusions \& Future Work}\label{sec:conclusion}
In this paper, we presented a comprehensive approach to MCPP on 2D grid graphs, addressing key challenges including grids with incomplete quadrant coarsening, non-uniform traversal costs, inter-robot conflicts, and turning costs in large-scale, real-world scenarios. 
Our contributions began with the ESTC paradigm, designed to overcome limitations of standard STC-based paradigms by enabling path generation on any grid structure, even those with incomplete quadrant coarsening. We further enhanced ESTC with two local optimizations, improving both path quality and computational efficiency.
Building on ESTC, we developed LS-MCPP, a local search framework that integrates ESTC to iteratively refine and optimize coverage paths, yielding high-quality solutions for MCPP on arbitrary grid graphs.
To address inter-robot conflicts, we introduced a new MAPF variant as a post-processing step for MCPP. This approach marks the first integration of MAPF techniques into MCPP, bridging two critical fields in multi-robot coordination. Our solver effectively resolves conflicts while preserving solution quality, demonstrating robust performance even in complex environments with up to 100 robots.
Our experimental results confirm that the proposed pipeline is scalable, robust, and capable of delivering high-quality solutions, making it suitable for a wide range of real-world multi-robot applications across diverse environments. This work lays a strong foundation for advancing multi-robot systems with implications for practical deployments in real-world applications.
Future work will explore encoding action costs on grid graphs with more precise action dynamics, further accelerating the local search through parallelization and machine learning techniques, and integrating online replanning to handle real-time changes and execution deviations.

\section*{Acknowledgement}
This work was supported by the NSERC under grant number RGPIN2020-06540 and a CFI JELF award.
We sincerely thank Chuxuan Zhang and Dr. Angelica Lim from SFU ROSIE lab for providing the two Pepper robots, Cayenne and Cumin, for our experiments.

\bibliographystyle{IEEEtran}
\bibliography{ref}

\end{document}